\documentclass{article}
 \usepackage{todonotes}

\usepackage[accepted]{icml2026}
\usepackage{graphicx}
\newcommand{\fob}{\tbZ}
 \newcommand{\babi}{\overleftarrow{b}}
  \newcommand{\foG}{\mathcal{L}^\rmM}
 
 \newcommand{\fobi}{\overrightarrow{b}}
 \newcommand{\foB}{\overrightarrow{B}}
 \newcommand{\baB}{\overleftarrow{B}}

    \def\bbbb{\mathrm{b}\bbb}

  \def\rmR{\mathrm{R}}

\usepackage{microtype}
\usepackage{graphicx}
\usepackage{subcaption}
\usepackage{booktabs} 
\usepackage{hyperref}

\usepackage{amsmath}
\usepackage{amssymb}
\usepackage{mathtools}
\usepackage{amsthm}

\usepackage[T1]{fontenc}      
\usepackage[utf8]{inputenc}   
\usepackage{float}
\usepackage{booktabs}
\usepackage{hyperref}
\usepackage{amsmath,amsthm}
\usepackage{mathtools}
\usepackage{amssymb,mathrsfs} 
\usepackage{amssymb}
\usepackage{amsfonts}
\usepackage{upgreek}
\usepackage{nicefrac}
\usepackage{graphicx}
 \usepackage{color}
 \usepackage{algorithm2e}
\usepackage{stmaryrd}
\DeclareMathAlphabet{\mathpzc}{OT1}{pzc}{m}{it}
\usepackage[inline]{enumitem}
\usepackage{url}
\usepackage{graphicx} 
\usepackage{xcolor}
\usepackage{bbm,bm}

\usepackage{ifthen}
\usepackage{xargs}
\usepackage{caption}
\usepackage{subcaption}

\hypersetup{
  colorlinks = true,
  citecolor = blue,
}

\usepackage{aliascnt}
\usepackage{cleveref}
\usepackage{autonum}
\makeatletter
\newtheorem{theorem}{Theorem}
\crefname{theorem}{theorem}{Theorems}
\Crefname{Theorem}{Theorem}{Theorems}

\newtheorem{lemma}{Lemma}
\crefname{lemma}{lemma}{lemmas}
\Crefname{Lemma}{Lemma}{Lemmas}

\newtheorem{corollary}{Corollary}
\crefname{corollary}{corollary}{corollaries}
\Crefname{Corollary}{Corollary}{Corollaries}

\newtheorem{proposition}{Proposition}
\crefname{proposition}{proposition}{propositions}
\Crefname{Proposition}{Proposition}{Propositions}

\crefname{example}{example}{examples}
\Crefname{Example}{Example}{Examples}

\crefname{figure}{figure}{figures}
\Crefname{Figure}{Figure}{Figures}

\newtheorem{assumption}{\textbf{H}\hspace{-3pt}}
\Crefname{assumption}{\textbf{H}\hspace{-3pt}}{\textbf{H}\hspace{-3pt}}
\crefname{assumption}{\textbf{H}}{\textbf{H}}

\usepackage{version}
\excludeversion{commenta}

\theoremstyle{remark}
\newtheorem{remark}{Remark}
\crefname{remark}{remark}{remarks}
\Crefname{Remark}{Remark}{Remarks}

\theoremstyle{definition}

\crefname{definition}{definition}{definitions}
\Crefname{Definition}{Definition}{Definitions}

\def\msa{\mathsf{A}}

\def\mse{\mathsf{E}}

\def\mcbb{\mathcal{B}}  
\newcommand{\mcb}[1]{\mathcal{B}(#1)}

\def\mcp{\mathcal{P}}

\def\mcp{\mathcal{P}}

\def\rset{\mathbb{R}}
\def\rsetd{\mathbb{R}^d}

\def\nset{\mathbb{N}}

\def\mrl{\mathrm{L}}
\def\rml{\mathrm{L}}
\def\rmd{\mathrm{d}}

\def\rmC{\mathrm{C}}

\newcommandx{\functionspace}[2][1=+]{\mathbb{F}_{#1}(#2)}

\newcommandx{\VarDeux}[3][3=]{\operatorname{Var}^{#3}_{#1}\left\{#2 \right\}}

\newcommand{\1}{\mathbbm{1}}

\newcommand{\LeftEqNo}{\let\veqno\@@leqno}

\newcommand{\N}{\ensuremath{\mathbb{N}}}

\newcommand{\PE}{\mathbb{E}}
\newcommand{\PP}{\mathbb{P}}

\newcommandx{\Vnorm}[2][1=V]{\| #2 \|_{#1}}
\newcommandx{\VnormEq}[2][1=V]{\left\| #2 \right\|_{#1}}
\newcommandx{\norm}[2][1=]{\ifthenelse{\equal{#1}{}}{\left\Vert #2 \right\Vert}{\left\Vert #2 \right\Vert^{#1}}}
\newcommandx{\normLigne}[2][1=]{\ifthenelse{\equal{#1}{}}{\Vert #2 \Vert}{\Vert #2\Vert^{#1}}}

\newcommand{\ps}[2]{\left\langle#1,#2 \right\rangle}

\newcommandx\probaMarkovTilde[2][2=]
{\ifthenelse{\equal{#2}{}}{{\widetilde{\mathbb{P}}_{#1}}}{\widetilde{\mathbb{P}}_{#1}\left[ #2\right]}}

\newcommand{\bigO}{\ensuremath{\mathcal O}}

\def\ie{\textit{i.e.}}

\def\eqsp{\;}

\newcommand{\ccint}[1]{\left[#1\right]}

\newcommandx{\weight}[2][2=n]{\omega_{#1,#2}^N}

\newcommandx\sequence[3][2=,3=]
{\ifthenelse{\equal{#3}{}}{\ensuremath{\{ #1_{#2}\}}}{\ensuremath{\{ #1_{#2}, \eqsp #2 \in #3 \}}}}
\newcommandx\sequenceD[3][2=,3=]
{\ifthenelse{\equal{#3}{}}{\ensuremath{\{ #1_{#2}\}}}{\ensuremath{( #1)_{ #2 \in #3} }}}

\newcommandx{\sequencen}[2][2=n\in\N]{\ensuremath{\{ #1_n, \eqsp #2 \}}}
\newcommandx\sequenceDouble[4][3=,4=]
{\ifthenelse{\equal{#3}{}}{\ensuremath{\{ (#1_{#3},#2_{#3}) \}}}{\ensuremath{\{  (#1_{#3},#2_{#3}), \eqsp #3 \in #4 \}}}}
\newcommandx{\sequencenDouble}[3][3=n\in\N]{\ensuremath{\{ (#1_{n},#2_{n}), \eqsp #3 \}}}

\newcommand{\opnorm}[1]{{\left\vert\kern-0.25ex\left\vert\kern-0.25ex\left\vert #1 
    \right\vert\kern-0.25ex\right\vert\kern-0.25ex\right\vert}}

\newcommand{\moment}[2][1=8]{\mathbf{m}_{#1}[{#2}]}

\def\Id{\operatorname{Id}}

\newcommandx{\CPE}[3][1=]{{\mathbb E}_{#1}\left[\left. #2 \, \middle \vert \, #3 \right. \right]} 

\newcommandx{\CPELigne}[3][1=]{{\mathbb E}_{#1}[\left. #2 \,  \vert \, #3 \right. ]} 

\newcommandx{\CPVar}[3][1=]{\mathrm{Var}^{#3}_{#1}\left\{ #2 \right\}}
\newcommand{\CPP}[3][]
{\ifthenelse{\equal{#1}{}}{{\mathbb P}\left(\left. #2 \, \right| #3 \right)}{{\mathbb P}_{#1}\left(\left. #2 \, \right | #3 \right)}}

\newcommandx{\osc}[2][1=]{\mathrm{osc}_{#1}(#2)}

\def\Id{\operatorname{Id}}

\def\transpose{\operatorname{T}}

\newcommand\coupling[2]{\Gamma(\mu,\nu)}

\renewcommand{\geq}{\geqslant}

\def\Leb{\mathrm{Leb}}

\newcommand{\wasserstein}{\mathscr{W}}

\def\wiener{\mathbb{W}}

\def\rmM{\mathrm{M}}

\def\gauss{\mathrm{N}}

\def\bfo{\mathbf{o}}

\newcommandx{\Voi}[1][1=i]{\mathfrak{V}_{\bfo,#1}}
\newcommandx{\Vlyapc}[2][1=\bfo,2=i]{\mathfrak{V}_{#1,#2}}

\def\Wlyap{\mathfrak{W}}
\def\bfomega{\boldsymbol{\omega}}

\newcommandx{\Woi}[1][1=i]{\Wlyap_{\bfomega,#1}}
\newcommandx{\Wlyapc}[2][1=\bfomega,2=i]{\Wlyap_{#1,#2}}

\def\bfA{\mathbf{A}}

\newcommand{\app}{X^{\theta^{\star}}}

 \newcommand{\foX}{\overrightarrow{X}}
 
 \newcommand{\baX}{\overleftarrow{X}}

 \newcommand{\KL}{\mathrm{KL}}

\newcommand{\applaw}{p^{\theta^{\star}}}

\def\wiener{\mathbb{W}}

\def\rmI{\mathrm{I}}

\def\thetas{\theta^{\star}}
\def\div{\operatorname{div}}

\def\nustar{\nu^{\star}}

\def\bbb{\mathbb{B}}

\def\tbZ{\tilde{\beta}}
\newcommandx{\inter}[2][2=]{\mathbb{I}_{#2}(#1)}
\newcommandx{\interM}[2][2=]{\mathbb{M}_{#2}(#1)}

\def\XintM{X^{\rmM}}

\def\eqlaw{\stackrel{\text{dist}}{=}}
\icmltitlerunning{KL and Wasserstein bounds for DFMs}

\begin{document}
\twocolumn[
  \icmltitle{Diffusion Flow Matching: Dimension-Improved \\
    KL Bounds and Wasserstein Guarantees}
\icmlsetsymbol{equal}{*}
\begin{icmlauthorlist}
    \icmlauthor{Marta Gentiloni Silveri}{yyy}
     \icmlauthor{Alain Durmus}{yyy}
     \icmlauthor{Giovanni Conforti}{zzz}
  \end{icmlauthorlist}

  \icmlaffiliation{yyy}{Ecole Polytechnique, Massy Palaiseau, France}
   \icmlaffiliation{zzz}{Università degli Studi di Padova, Padua, Italy}

  \icmlcorrespondingauthor{Marta Gentiloni Silveri}{marta.gentiloni-silveri@polytechnique.edu}
  \icmlcorrespondingauthor{Alain Durmus}{alain.durmus@polytechnique.edu}
  \icmlcorrespondingauthor{Giovanni Conforti}{giovanni.conforti@math.unipd.it}

 \icmlkeywords{Diffusion Flow Matching, Kullback–Leibler divergence, Wasserstein distance, Convergence Guarantees, ICML}

  \vskip 0.3in
]

\printAffiliationsAndNotice{} 

\begin{abstract}

    Diffusion Flow Matching (DFM) has recently emerged as a versatile framework for generative modeling, yet its theoretical convergence properties remain only partially understood. In this work, we provide refined and novel convergence guarantees for Brownian motion based DFMs, focusing on the  discretization error. Our analysis is conducted under the Kullback–Leibler (KL) divergence and the 2-Wasserstein distance. Under finite-moment conditions and a mild score integrability assumption, we derive KL convergence bounds with improved dimensional dependence compared to prior work, achieving, up to our knowledge, state-of-the-art scaling under minimal conditions. We further extend the analysis to the 2-Wasserstein distance: under an additional first-order score integrability assumption and a weak log-concavity condition, we obtain convergence guarantees with dimensional dependence consistent with the KL case.   
\end{abstract}

\section{Introduction}
A central problem in machine learning and statistics is the generation of new samples from a target distribution that is only accessible through finite data. Generative modeling addresses this challenge by learning a mechanism that transforms a simple and tractable base distribution into the data distribution of interest. Deep generative models have been successfully applied in a wide range of domains such as image and video synthesis \citep{ho2020denoising, nichol2021improved}, speech and audio generation \citep{kong2020diffwave}, molecular and material design \citep{liu2023generativeDiffGraphs, Zang2020MoFlow, FlowMol3_2025}, and medical imaging \citep{loo2025diffusionReview}, owing to their ability to learn and simulate complex data-generating processes from observations alone.

Among the broad class of generative approaches, continuous-time generative models have recently emerged as a particularly powerful and conceptually elegant paradigm. These models describe the transformation between probability distributions through differential equations or stochastic dynamics evolving over time, enabling fine-grained control over the generation process and providing a natural connection to tools from stochastic analysis and optimal transport. Examples include score-based generative models (SGMs) that leverage stochastic differential equations \cite{ho2020denoising,song2020generative,song2019generative}, and probability flow ordinary differential equations \cite{chen2023the,rombach2022high,ramesh2022hierarchical,popov2021grad} that enable deterministic sampling.

Within this landscape, Flow Matching (FM) models have gained significant attention as a unifying framework for constructing continuous-time transports between probability distributions \citep{peluchetti2022nondenoising,lipman2024flowMatchingGuide,albergo2022building,albergo2023stochastic,lipman2023flow,liu2022rectified,Liu2022FlowSA}. The key idea underlying FM is to define a coupling $\pi$ between a base distribution $\mu$ and a target distribution $\nustar$, and to introduce an interpolating process, referred to as an \emph{interpolant}, that bridges the two distributions over a finite time horizon. The dynamics of this interpolant induce a time-dependent velocity field that can be learned from samples, thereby specifying a transport map from $\mu$ to $\nustar$. 

When the interpolant follows deterministic dynamics, the resulting model corresponds to a deterministic FM formulation. Allowing for stochasticity in the interpolant leads instead to \emph{Diffusion FM}  (DFM) models. While FM and its stochastic formulation offer greater flexibility and robustness, they also introduce substantial technical challenges. In general, the interpolant does not define a Markov diffusion process and therefore cannot be directly characterized by a stochastic differential equation. To address this issue, DFM relies on the construction of a \emph{Markovian projection}: a diffusion process that matches the marginal distributions of the interpolant at every time. The drift of this mimicking diffusion satisfies a regression identity and can be efficiently approximated using a neural network trained on samples from the interpolant. Once learned, the diffusion process can be simulated using standard numerical schemes, such as the Euler–Maruyama method, to generate samples from the learned distribution.


Despite promising empirical results \cite{albergo2023stochastic} and its conceptual appeal, the theoretical foundations of Diffusion Flow Matching remain comparatively underdeveloped. Existing analyses are largely confined to deterministic FM models, while rigorous convergence guarantees for DFMs, particularly in terms of quantitative error bounds and distributional metrics, are still scarce. This gap motivates a deeper theoretical investigation of DFM models and their convergence properties.

\paragraph*{Our contribution}
In this work, we analyze DFM  models based on $d$-dimensional Brownian bridge. We establish theoretical guarantees for convergence to the target distribution, both in Kullback–Leibler ($\KL$) divergence and in Wasserstein-2 ($\wasserstein_2$) distance, under standard and mild assumptions on the data. Our results extend and improve upon prior work by sharpening the dimensional dependence of the $\KL$ guarantees and by establishing $\wasserstein_2$ bounds accounting for all the sources of error. Our main contributions are summarized as follows:
\begin{enumerate}[wide, labelwidth=!, labelindent=0pt,label=(\arabic*)]
    \item \textbf{$\KL$ convergence bounds.}
    \begin{enumerate}[wide, labelwidth=!, labelindent=0pt,label=$\bullet$]
    \item \textbf{Without early stopping and constant step-size.} We derive in \Cref{theo_no_early} an improved explicit upper bound on the $\KL$ divergence between the target distribution $\nustar$ and the DFM output without early stopping, under standard assumptions. In particular, we suppose that the two marginals $\mu$ and $\nustar$ admit finite $8$-th moment (\Cref{ass_moment}) that the coupling $\pi$ admits a score with finite $8$-th moment  (\Cref{ass_score}), and standard $\mrl^2$ drift-approximation accuracy (\Cref{ass_drift_approx}) under the Markovian projection. The resulting bound achieves $\mathcal{O}(h)$ dependence on the time step size, while improving the dimensional scaling from $\mathcal{O}(d^4)$ to $\mathcal{O}(d^3)$ compared to prior works.
    \item \textbf{With early stopping and constant step-size.}
By assuming the same moment conditions on \(\mu\) and \(\nustar\) and approximation error of the drift, replacing the condition  \Cref{ass_score} on $\pi$ by the mild condition that the score of the conditional coupling $\pi_{0|1}$ admits finite $8$-th order moment \Cref{ass_score_conditioned}, we obtain in \Cref{theo_early} an explicit bound on the \(\KL\) divergence between a smoothed target distribution and the early-stopped DFM output. Our bound preserves the $\mathcal{O}(h)$ dependence and the improved \(\mathcal{O}(d^3)\) dimensional scaling of the non-early-stopped case. From this result, we deduce in \Cref{cor:fourier} that choosing $\delta = \bigO( \varepsilon^2/d)$ and $h= \bigO(\varepsilon^{10}/d^{7})$ yields a $\wasserstein_{2,\text{FM}}^2$-error of order $\mathcal{O}(\varepsilon^2)$.
\item \textbf{With early stopping and novel step-size schedule.}
In \Cref{theo_faster}, we establish faster $\KL$ convergence rates in the early stopping regime via a novel step-size schedule, while preserving the $\bigO(d^3)$ and $\mathcal{O}(h)$ dependences. The result relies only on the moment and drift approximation assumptions in \Cref{ass_moment,ass_drift_approx}, together with the integrability condition on the score of the conditional coupling $\pi_{0|1}$ (\Cref{ass_score_conditioned}). As a consequence, \Cref{cor:fourier_expo} yields improved complexity bounds in the Fortet–Mourier metric: choosing $\delta = \bigO( \varepsilon^2/d)$, it suffices to take $h= \tilde{\bigO}(\varepsilon^{2}/d^{3})$, where $\tilde{\bigO}$ hides logarithmic factors in $d$ and $1/\varepsilon$, to guarantee a $\wasserstein_{2,\text{FM}}^2$-error of order $\mathcal{O}(\varepsilon^2)$.
\end{enumerate}

    \item \textbf{$\wasserstein_2$ convergence bounds.}
    \begin{enumerate}[wide, labelwidth=!, labelindent=0pt,label=$\bullet$]
        \item \textbf{Without early stopping and constant step-size.} We establish \(\wasserstein_2\) convergence bounds for DFMs in the non-early-stopped regime in \Cref{theo_wasserstein_weak_log_concave}, applicable to a broad class of distributions. Under appropriate $\mrl^2$-approximation error for the drift \Cref{ass_drift_approx_wass}, moment conditions \Cref{ass_moment,ass_score}, a weak log-concavity assumption on the coupling $\pi$ (\Cref{ass_regularity}) and an integrability condition on the Jacobian of the score associated with $\pi$ (\Cref{ass_hessian}), we derive bounds scaling as \(\mathcal{O}(\sqrt{h})\) in the time step and \(\mathcal{O}(\sqrt{d^3})\) in the dimension. These rates are consistent with the corresponding $\KL$ guarantees.
Furthermore, in \Cref{cor_wass_weak_indep}, we show that our results apply in the case where $\pi$ is the independent coupling, provided that the scores of the marginals \(\mu\) and \(\nustar\), together with their Jacobians, are integrable, and weakly concave (\Cref{ass_indep_coupling}).
\end{enumerate}
\end{enumerate}

\paragraph{Notation}
Given a measurable space $(\mse, \mathcal{E})$, we denote by $\mathcal{P}(\mse)$ the set of probability measures of $\mse$. Also, given a topological space $(\mse, \tau)$, we use $\mathcal{B}(\mse)$ to denote the Borel $\sigma$-algebra on $\mse$. We denote by $\wiener= \rmC([0,1],\rset^d)$ the set of
continuous functions from $[0,1]$ to $\rset^d$ and we refer to it as  the Wiener space. 
We denote by $\Leb^d$ the Lebesgue measure on $\rset^d$. Given two probability measures $\mu, \nu \in \mathcal{P}(\rsetd)$, we denote by $\Pi(\mu,\nu)$ the set of couplings between $\mu$ and $\nu$, \ie , 
  $\xi\in\Pi(\mu,\nu)$ if and only if $\xi$ is a probability measure on $\rset^d\times\rset^d$ and $\xi(\msa \times \rset^d) = \mu(\msa)$ and $\xi(\rset^d\times \msa) = \nu(\msa)$ for all measurable $\msa \subseteq \rset^d$. The relative entropy (or $\KL$-divergence) of $\mu$ with respect to $\nu$ is defined by $\KL(\mu |\nu) := \int \log (\rmd \mu /\rmd \nu) \rmd\mu$ if $\mu$ is absolutely continuous with respect to $\nu$, and $\KL(\mu |\nu) := +\infty$ otherwise. If $\mu$ and $\nu$ have finite second moment, the $2-$Wasserstein distance is defined by $\wasserstein_2^2 (\mu, \nu) := \inf_{\xi \in \Pi(\mu,\nu)} \int \norm{x-y}^2 \rmd \xi(x,y)$ and the Fortet-Mourier distance of order $2$ is defined by $\wasserstein_{2,\text{FM}}^2(\mu,\nu)=\inf_{\pi\in\Pi(\mu,\nu)}\int \min\{\|x-y\|^2, 1\} \rmd \pi (x,y).$ Given $\xi\in\mathcal{P}(\rsetd)$ and $p\ge 1$, we denote by $\|\cdot\|_{\mathrm{L}^p(\xi)}:=(\int\|\cdot\|^p \rmd \xi)^{1/p}$ the $\mathrm{L}^p$-norm with respect to $\xi$. Given two real numbers $u,v\in \rset$, we write $u\lesssim v$ (resp. $u \gtrsim v$) to mean $u\le C v$ (resp. $u\ge C v$) for a universal constant $C>0$. 
Also, we denote by $\norm{x}$ the Euclidean norm of $x\in\rset^d$, by $\langle x, y\rangle$ the scalar product between $x,y \in\rset^d$, and  by $x^{\transpose}$ the transpose of $x$. Last, we use standard Big-$\bigO$ notation. 

\section{Diffusion Flow Matching}

In this section, we provide a concise yet self-contained overview of Brownian motion based Diffusion Flow Matching (DFM) models, following the probabilistic formulation introduced in recent works.
Recall that $\nustar \in \mathcal{P}(\rsetd)$ denote the target (or data) distribution and $\mu \in \mathcal{P}(\rsetd)$ denote the base (or prior) distribution.

As already highlighted, DFM is a procedure that designs a generative model able to produce new samples approximately distributed according to $\nustar$, by learning a stochastic transport from $\mu$ to $\nustar$ over a finite time horizon. To this end, DFM proceeds as follows:

\begin{enumerate}
    \item \textbf{Markovian projection.}  
    First, DFM constructs the Markovian projection $(X^{\mathrm M}_t)_{t\in\ccint{0,1}}$ of an interpolated process $(X^{\mathrm I}_t)_{t\in\ccint{0,1}}$, such that $X^{\mathrm M}_0$ and $X^{\mathrm M}_1$ have distributions $\mu$ and $\nu^\star$, respectively. The process $(X^{\mathrm M}_t)_{t\in\ccint{0,1}}$ is defined as the solution to a stochastic differential equation (SDE).
 
    \item \textbf{Approximation of the Markovian projection.}  
    While $(X^{\mathrm M}_t)_{t\in\ccint{0,1}}$ would yield an exact generative model mapping samples from $\mu$ to samples from $\nu^\star$ if the associated SDE could be solved exactly, this is generally infeasible in practice. Consequently, suitable numerical or modeling approximations must be introduced.
\end{enumerate}

We detail these two crucial stages in what follows.

\paragraph{Markovian projection.}
The core objects at the basis of this construction is a coupling $\pi$ between $\mu$ and $\nustar$, \ie, a probability measure on the product space $\rset^{2d}$ such that, for any $\msa \in\mcbb(\rset^d)$,
$\pi(\msa\times \rset^d) = \mu(\msa)$ and
$\pi(\rset^d\times \msa)= \nustar(\msa)$, and the transition of the Brownian bridge process that we now recall. It is well-known that if $(B_t)_{t \in [0,1]}$ is a Brownian motion starting from a distribution $\mu\in\mcp(\rset^d)$, then it induces a Markov kernel $\bbbb$ 
    which corresponds to a regular version of the  conditional distribution of the path
    $(B_t)_{t\in[0,1]}$ given $B_0$ and $B_1$, \ie, 
    \begin{equation}
     \PP[(B_t)_{t\in[0,1]} \in \msa | (B_0,B_1)] =  \bbbb((B_0,B_1),\msa) \eqsp,
    \end{equation}
    for any $\msa\in \mcb{\wiener}$. 
    The existence of this kernel is guaranteed, for instance, by Theorem~8.37 in \cite{klenke2013probability}. Then, by integrating these bridges against the coupling $\pi$, one defines the \emph{stochastic interpolant}
    \begin{align}\label{def:stochastic_interpolant}
        \inter{\pi,\bbb}(\msa)
        = \int \bbbb((x_0,x_1),\msa)\, \pi(\rmd x_0,\rmd x_1)\eqsp,
    \end{align}
    for any $\msa \in \mathcal{B}(\wiener)$.
    This construction corresponds to the law of a stochastic process
    $(X_t^{\rm I})_{t\in[0,1]}$, which defines a random path connecting the two marginals: by construction, the interpolant satisfies
    $X^{\rm I}_0 \sim \mu$ and $X^{\rm I}_1 \sim \nustar$.

    While $(X_t^{\rm I})_{t\in[0,1]}$ provides a conceptually meaningful interpolation between $\mu$ and $\nustar$, it is generally \emph{non-Markovian}.
    Indeed, its dynamics depend on the terminal value $X^{\rm I}_1$ and satisfies
    \begin{equation}\label{eq:sde_stochastic_interpolant}
        \rmd X^{\rm I}_t
        = 2 \nabla_x \log p_{1-t}(X^{\rm I}_1 \mid X^{\rm I}_t)\rmd t
        + \sqrt{2}\rmd W_t\eqsp,\quad t\in [0,1]\eqsp,
    \end{equation}with $X^{\rm I}_0\sim \mu.$
    Here $(W_t)_{t\in [0,1]}$ denotes a $d$-dimensional Brownian motion and $(y,x) \mapsto p_{s}(y|x)$ denotes the heat kernel, \ie, for any $x,y\in\rset^d$ and $0<s\le 1$
    \begin{equation}\label{def_heat_kernel}
 p_{s}(y|x) = \frac{1}{(4\pi s)^{d/2}} \exp \Big(-\frac{\norm{y-x}^2}{4s}\Big)\eqsp\eqsp.
\end{equation}
    We refer to Section~2.1 of \cite{gentiloni2024theoretical} for details.

    To obtain a tractable generative model, one constructs a Markov diffusion
    $(X_t^{\rm M})_{t\in[0,1]}$ that shares the same time marginals as the interpolant.
    This is achieved via the Markovian projection \cite{gyongy1986mimicking,krylov:1984}:
    $X^{\rm M}_0 \sim \mu$ and
    \begin{equation}\label{sde:markovian_proj}
        \rmd X_t^{\rm M}
        = \tbZ_t(X_t^{\rm M})\rmd t + \sqrt{2}\rmd W_t\eqsp,
        \quad t \in [0,1]\eqsp,
    \end{equation}
    where the mimicking drift is given by the conditional expectation
    \begin{align}\label{drift_as_conditional_expectation}
        \tbZ_t(x)
        &= \PE\big[
         2 \nabla_x \log p_{1-t}(X^{\rm I}_1 \mid X^{\rm I}_t)
        \,\big|\, X^{\rm I}_t = x
        \big]\\
        &=\PE\left[ \frac{X^\rmI_1-X^\rmI_t}{1-t} \Big| X_t^{\rm I}=x\right]\eqsp.
    \end{align}
    The resulting process satisfies
    \begin{align}\label{equality_in_law_of_marginals}
        X^{\rm I}_t \eqlaw X^{\rm M}_t \eqsp,
    \end{align}
    for all $t \in [0,1]$ \cite{gentiloni2024theoretical,silveri2025exponential},
    and therefore constitutes an ideal diffusion model transporting $\mu$ to $\nustar$.

While the Markovian projection provides an ideal generative model, it is not directly accessible. First, the mimicking drift defined in \eqref{drift_as_conditional_expectation} is intractable and even using some approximation, the continuous time  SDE  \eqref{sde:markovian_proj} cannot be simulated exactly. As a result, turning this construction into a workable model requires overcoming these practical limitations.

\paragraph{Approximation of the Markovian projection.}
    The first challenge in approximating the Markovian projection is that of approximating the mimicking drift $\tbZ_t$.
    Since, for any $t\in [0,1]$, $\tbZ_t$ writes as a conditional expectation \eqref{drift_as_conditional_expectation}, then,
by Corollary 8.17 in \cite{klenke2013probability} and \eqref{equality_in_law_of_marginals}, $\tbZ_t$ solves the regression problem:
    \begin{equation}
         \arg\min_{f}\PE\Big[
        \norm{
        f(t,X^{\rm M}_{t})
        - \tbZ_{t}(X^{\rm M}_{t})
        }^2
        \Big]\eqsp.
    \end{equation}
    Therefore, for any $t\in [0,1]$, $\tbZ_t$ can be estimated using a family of neural networks
    $\{x \mapsto s_\theta(t,x)\}_{\theta \in \Theta}$
    and minimizing the $\mrl^2$ regression loss
\begin{equation}\label{minimization_problem}
        \theta \mapsto
        \PE\Big[
        \norm{
        s_\theta(t,X^{\rm M}_{t})
        - \tbZ_{t}(X^{\rm M}_{t})
        }^2
        \Big]\eqsp.
    \end{equation}
    
    The trained network $s_{\theta^\star}(t,\cdot)$ then serves as proxy of $\tbZ_{t}$. However, the resulting diffusion which reads:
\begin{align}
    \rmd X^{\mathrm{NN}}_t=s_{\theta^\star}(t, X^{\mathrm{NN}}_t)\rmd t+ \sqrt{2} \rmd W_t\eqsp, \quad t\in [0,1]\eqsp,
\end{align}with $X^{\mathrm{NN}}_0\sim \mu$,
    cannot be solved in general and we have to rely on numerical discretization. We make use of the Euler-Maruyama scheme, \ie, for a choice of sequence of step sizes $\{h_k\}_{k=1}^N$, $N \geq 1$, and the corresponding time discretization $t_k= \sum_{i=1}^{k} h_i $, such that $t_0=0$ and $t_N = 1$, we
define the continuous process $(\app_t)_{t \in [0,1]}$, recursively on the intervals $[t_k, t_{k+1}]$ by
\begin{equation}\label{model}
    \rmd \app_t=s_{\theta^\star}(t_k, \app_{t_k})\rmd t +\sqrt{2} \rmd W_t\eqsp, \quad t\in [t_k,t_{k+1}]\eqsp, 
\end{equation}

with $\app_0\sim \mu$. Finally, this dynamics is used to define the generative model associated with the DFM procedure. In particular, approximate samples from $\nustar$ are obtained by simulating trajectories of $(\app_t)_{t\le 1-\delta}$ with $\delta \geq 0$. When $\delta=0$, the dynamics is run until its terminal time and no early stopping is performed. In contrast, choosing $\delta>0$ results in an early stopping of the dynamics, which may be beneficial in practice to mitigate numerical instabilities near the terminal time. 

In short, DFM generates samples by simulating an Euler–Maruyama discretization of a neural-network approximation of the Markovian projection associated with the stochastic interpolant \eqref{model}. The simulation is run up to time $t = 1 - \delta$, for some $\delta \ge 0$.


\section{Main Results}

In this section, we present our main theoretical contributions: \emph{non-asymptotic convergence guarantees} for Brownian-motion-based Diffusion DFM~\eqref{model}.
We analyze convergence in two metrics: Kullback--Leibler divergence and Wasserstein-2 distance.
Our results are established under mild moment and regularity conditions on the distributions $\mu$, $\nustar$, and the coupling $\pi$, together with standard $\mrl^2$ drift-approximation assumptions common in the literature on diffusion and score-based generative models~\cite{chen2022sampling,chen2023improved,gentiloni2024theoretical,gao2023wasserstein}.

\subsection{Convergence in Kullback--Leibler Divergence}
\label{section_convergence_bounds_kl}

We first focus on convergence in $\KL$ divergence, which provides a strong notion of similarity between distributions and is widely used in theoretical analyses of diffusion models~\cite{chen2023improved,conforti2023score,silveri2025exponential}.   We state next two assumptions that we will be made in all of our results.

As is standard, we first assume that the learned drift approximates the true drift with $\varepsilon^2$ accuracy in $\mrl^2$ norm.

\begin{assumption}[Drift approximation]\label{ass_drift_approx}
There exists $\theta^\star\in\Theta$ and $\varepsilon^2>0$ such that
\[
\sum_{k=0}^{N-1} h_{k+1}
\PE\big[\|s_{\theta^\star}(t_k, \XintM_{t_k}) - \fob_{t_k}(\XintM_{t_k})\|^2 \big]
\le \varepsilon^2 \eqsp.
\]
\end{assumption}

Furthermore, we impose the following moment conditions.
For $m\in\nset$, $\zeta\in\mcp(\rset^m)$, and $p\ge1$, we denote by $\moment[p]{\zeta} = \int \|x\|^p \rmd\zeta(x)$ the $p$-th moment. When $\zeta$ has a density with respect to $\Leb^m$, we identify it with the density itself.

\begin{assumption}[Moment condition]\label{ass_moment}
It holds $\moment[8]{\mu} + \moment[8]{\nustar}<+\infty$.
\end{assumption}

\paragraph{KL convergence without early stopping and constant step-size.}
In the full regime setting, we also impose standard integrability conditions on the score of the coupling, which are routinely assumed in analyses of diffusion-based generative models \cite{conforti2023score,gentiloni2024theoretical}.
\begin{assumption}[Score integrability of the coupling]\label{ass_score}
The coupling $\pi$ is absolutely continuous with respect to $\Leb^{2d}$ with positive density, $\log\pi$ is $\rmC^1$, and $ \|\nabla \log \pi \|_{\mrl^8(\pi)}<+\infty.$
\end{assumption}

Under these assumptions, we can quantify how close the law of the generated samples is to the target $\nustar$ in $\KL$ divergence.
\begin{theorem}\label{theo_no_early}
Let $\{t_k\}_{k=0}^{N_h}$ be a uniform partition of $[0,1]$ with step size $h=1/N_h>0$.
Under \Cref{ass_drift_approx,ass_moment,ass_score}, if $\moment[8]{\mu}, \moment[8]{\nustar}\lesssim d^4$, denoting by $\nu^{\theta^\star}_1$ the law of $\app_1$, we have
\begin{align}\label{convergence_bound_no_early}
\KL(\nustar | \nu^{\theta^\star}_1)&\lesssim \varepsilon^2 +h(h^{1/2}+1) \left(d^2+\norm{\nabla\log \pi}_{\mrl^8(\pi)}^4\right)d\eqsp.
\end{align}
\end{theorem}
The bound provided in \Cref{theo_no_early} has two additive contributions: the first term, $\varepsilon^2$, corresponds to the drift-approximation error; while the second term accounts for the discretization error and scales as $\bigO(h)$.
Importantly, the bound scales cubically with the dimension $d$, improving upon previous $d^4$-scaling results~\cite{gentiloni2024theoretical}. 
We refer to \Cref{sec:literature} for a detailed literature comparison. Note that, to obtain a $\KL$ error of order $\mathcal{O}(\varepsilon^2)$, one needs
\begin{align}
N_h& = \frac{1}{\varepsilon^2}\left(d^2+\norm{\nabla\log \pi}_{\mrl^8(\pi)}^4\right)d \eqsp,
\end{align}
discretization points.
\begin{remark}\label{rem:change_of_time_horizon}
For simplicity, we set the time horizon to $T=1$. For an arbitrary $T>0$, the bound in \Cref{theo_no_early} remains valid with a Brownian reference on $[0,T]$, up to a multiplicative factor $\max\{1,T^2\}$ in the discretization error.
\end{remark}

\paragraph{KL convergence with early stopping and constant step-size.}
If the process is stopped before reaching $t=1$, we can relax \Cref{ass_score} and refine the bound. 

More precisely, let $\pi_{0|1}$ denote a regular conditional distribution of $X^{\rmI}_0$ given $X^{\rmI}_1$. We assume the following integrability condition.

\begin{assumption}[Score integrability of the conditional coupling]\label{ass_score_conditioned}
The conditional coupling $\pi_{0|1}$ is absolutely continuous with respect to $\Leb^d$, with strictly positive density. Moreover, $\log \pi_{0|1}$ is $\rmC^1$ and satisfies $ \|\nabla \log \pi_{0|1}  \|_{\mrl^8(\pi_{0|1} )}<+\infty.$
\end{assumption}
We remark that \Cref{ass_score_conditioned} is substantially weaker than \Cref{ass_score}.
For instance, when the coupling is chosen to be independent, that is $\pi=\mu\otimes\nustar$, the conditional distribution $\pi_{0|1}$ coincides with the prior distribution $\mu$, and \Cref{ass_score_conditioned} reduces to an integrability condition on the score of $\mu$: $\|\nabla \log \mu \|_{\mrl^8(\mu)} < +\infty.$
Consequently, the only assumption involving the target distribution $\nustar$ is the moment condition \Cref{ass_moment}.


Under this milder assumption, we are able to derive the following result.
\begin{theorem}\label{theo_early}
Fix $0<\delta<1/2$. Let $\{t_k\}_{k=0}^{N_h}$ be a uniform partition of $[0,1]$ with step size $h=1/N_h>0$.
Under \Cref{ass_drift_approx,ass_moment,ass_score_conditioned}, if $\moment[8]{\mu}, \moment[8]{\nustar}\lesssim d^4$, then, for $\nustar_{1-\delta}$ and $\nu^{\theta^\star}_{1-\delta}$ denoting the laws of $\XintM_{1-\delta}$ and $\app_{1-\delta}$, we have
\begin{align}\label{convergence_bound_early}
\KL(\nustar_{1-\delta} | \nu^{\theta^\star}_{1-\delta})&\lesssim  \varepsilon^2+ h (h^{1/2}+1) \\
&\quad \cdot\Bigg(\frac{d^2}{\delta^4} + \|\nabla\log \pi_{0|1}\|_{\mrl^8(\pi_{0|1})}^4 \Bigg)d\eqsp.
\end{align}
\end{theorem}
Consistently with Theorem~\ref{theo_no_early}, the dependence on the dimension remains cubic in $d$. As a consequence, our result improves upon Theorem~3 in \cite{gentiloni2024theoretical}. In this setting, achieving a $\KL$ error of order $\mathcal{O}(\varepsilon^2)$ is ensured by choosing
    \begin{align}
        N_h&=\frac{1}{\varepsilon^2}\Bigg(\frac{d^2}{\delta^4} + \|\nabla\log \pi_{0|1}\|_{\mrl^8(\pi_{0|1})}^4 \Bigg)d\eqsp.
    \end{align}
An extensive comparison with the related literature is provided in \Cref{sec:literature}.

\begin{corollary}\label{cor:fourier}
    Fix $\delta = \bigO( \varepsilon^2/d)$ and $h= \bigO(\varepsilon^{10}/d^{7})$. Let $\{t_k\}_{k=0}^{N_h}$ be a uniform partition of $[0,1]$ with step size $h=1/N_h>0$. Under \Cref{ass_drift_approx,ass_moment,ass_score_conditioned}, if $\moment[8]{\mu}, \moment[8]{\nustar}\lesssim d^4$, then, for $\nu^{\theta^\star}_{1-\delta}$ denoting the law of  $\app_{1-\delta}$, we have
    \begin{equation}
        \wasserstein_{2,\text{FM}}^2(\nustar |\nu^{\thetas}_{1-\delta} )\lesssim \mathcal{O}(\varepsilon^2)\eqsp.
        \end{equation}
    \end{corollary}

 This result ensures a computational complexity $\mathcal{O}(d^7\varepsilon^{-10}).$
\paragraph{KL convergence with early stopping and novel step-size schedule.}
In the early stopping regime, we show that an appropriately designed step-size schedule yields an accelerated rate of convergence in $\KL$. Our result holds under \Cref{ass_moment,ass_drift_approx,ass_score_conditioned}.
\begin{theorem}\label{theo_faster}
    Fix $0<\delta<1/2$. Let $\{t_k\}_{k=0}^{M_h+N}$ be a partition of $[0,1]$ with step sizes $\{h_k\}_{k=1}^{M_h+N}$ such that $h_k=h=1/(2M_h)$ for $k\le M_h$ and $h_{k}=h\min\{t_k, 1-t_k\}$ for $M_h<k\le M_h+N$. Under \Cref{ass_moment,ass_drift_approx,ass_score_conditioned}, denoting by $\nustar_{1-\delta}$ and $\nu^{\theta^\star}_{1-\delta}$ the laws of $\XintM_{1-\delta}$ and $\app_{1-\delta}$, we have
\begin{align}\label{convergence_bound_no_early}
\KL(\nustar_{1-\delta} | \nu^{\theta^\star}_{1-\delta})&\lesssim \varepsilon^2+ h d^3\log\frac{1}{\delta}+h(h^{1/2}+1)\\
&\quad \cdot\Big(d^2+ \|\nabla\log \pi_{0|1}\|_{\mrl^8(\pi_{0|1})}^4 \Big)d\eqsp.
\end{align}
\end{theorem}
Note that, to obtain a $\KL$ error of order $\mathcal{O}(\varepsilon^2)$, one needs
\begin{align}
M_h& = \frac{1}{2\varepsilon^2}\left[\Big(d^2+ \|\nabla\log \pi_{0|1}\|_{\mrl^8(\pi_{0|1})}^4 \Big)d +d^3\log\frac{1}{\delta}\right],
\end{align}
and $N=2 M_h \log(1/\delta)$.
\begin{corollary}\label{cor:fourier_expo}
    Fix $\delta = \bigO( \varepsilon^2/d)$ and $h= \tilde{\bigO}(\varepsilon^{2}/d^{3})$, where $\tilde{\bigO}$ hides logarithmic factors in $d$ and $1/\varepsilon$. Let $\{t_k\}_{k=0}^{M_h+N}$ be a partition of $[0,1]$ with step sizes $\{h_k\}_{k=1}^{M_h+N}$ such that $h_k=h$ for $k\le M_h$ and $h_{k}=h\min\{t_k, 1-t_k\}$ for $M_h<k\le M_h+N$. Under \Cref{ass_moment,ass_drift_approx,ass_score_conditioned}, if $\moment[8]{\mu}, \moment[8]{\nustar}\lesssim d^4$, then for $\nu^{\theta^\star}_{1-\delta}$ denoting the law of  $\app_{1-\delta}$, we have
    \begin{equation}
        \wasserstein_{2,\text{FM}}^2(\nustar,\nu^{\thetas}_{1-\delta} )\lesssim \mathcal{O}(\varepsilon^2)\eqsp.
        \end{equation}
    \end{corollary}
 This result ensures a computational complexity $\tilde{\mathcal{O}}(d^3\varepsilon^{-2}).$

\subsection{Convergence in Wasserstein-2 Distance}

We now turn our attention to convergence in \emph{Wasserstein-2 distance}, which is a natural metric for evaluating the quality of generated samples in generative modeling \cite{silveri2025beyond,strasman2024analysis,strasmanwasserstein}.

In this section, we introduce a modified $\mrl^2$ drift approximation condition, tailored to the Wasserstein setting:
\begin{assumption}[Drift approximation for Wasserstein analysis]\label{ass_drift_approx_wass}
There exist $\theta^\star$ and $\varepsilon>0$ such that for any $k\in \{0,...,N-1\}$
\begin{align}
    \PE\left[\left\|\fob_{t_k}(\app_{t_k})-s_{\theta^\star}(t_k,\app_{t_k})\right\|^2\right]^{1/2}\le \varepsilon\eqsp.
\end{align}
\end{assumption}

Assumptions of this form are standard in the analysis of SGMs under Wasserstein metrics (see, e.g.,~\cite{gao2023wasserstein,bruno2023diffusion,strasman2024analysis,silveri2025beyond,strasmanwasserstein}).
An important feature of \Cref{ass_drift_approx_wass} is that the expectation is taken over the trajectory of the generative process $(\app_{t_k})_{k=0}^{N}$.

In addition, we assume that $\pi$ satisfies a weak log-concavity condition in the sense of \cite{conforti2022weak}, together with suitable integrability assumptions on the Jacobian of its associated score function. 

Let $\beta:\rset^d\to \rset$ be a smooth potential function. Its \emph{weak convexity profile} is defined as
\begin{align}\label{eq:def:weak_convexity_profile}
\kappa_\beta(r) = \inf\left\{ \frac{\langle \nabla \beta(x) - \nabla \beta(y), x-y\rangle}{\|x-y\|^2} : \|x-y\|=r \right\}\eqsp,
\end{align}for any $r>0$.
This function quantifies non-uniform convexity lower bounds, allowing one to weaken classical convexity conditions.
Furthermore, for any $M\ge 0$ and $r>0$, define
\[
f_M(r) = 2\sqrt{M}\,\tanh\big(r\sqrt{M}/2\big)\eqsp\eqsp.
\]

The weak convexity profile \eqref{eq:def:weak_convexity_profile} is widely used in the analyses of diffusion-based generative models.
Recent results establish state-of-the-art $\KL$ and $\wasserstein_2$ convergence guarantees for these models under such weak regularity conditions~\cite{silveri2025beyond,bruno2025wasserstein,kremling2025non,silveri2025exponential}.

We assume that $\pi$ is weakly log-concave:
\begin{assumption}[Weak log-concavity of the coupling]\label{ass_regularity} $\pi\in\mcp(\rset^d)$ is absolutely continuous with respect to $\Leb^{2d}$, $\log\zeta$ is $\mathcal{C}^1$, and satisfies for some $\alpha_{\pi}>0,M_{\pi}\ge0$ and any $r>0,$ \begin{align} \kappa_{-\log\pi}(r)\ge \alpha_{\pi}-r^{-1}f_{M_{\pi}}(r)\eqsp.\end{align} 
\end{assumption}
Strongly log-concave distributions satisfy \Cref{ass_regularity} as a special case choosing $M_{\zeta} =0$. Furthermore, perturbations of strongly log-concave distributions are weakly log-concave:
for instance, if $-\log \xi = V + W$ with $V$ strongly convex and $W$ smooth with Lipschitz gradient, then $\xi$ is weakly log-concave.
This includes classical double-well potentials, which are widely used in physics and statistics. Finally, it is well-known that Gaussian mixtures satisfy \Cref{ass_regularity}.\\

Moreover, we assume that the Jacobian of the score function associated to $\pi$ is integrable:
\begin{assumption}[First-order score integrability of the coupling]\label{ass_hessian}
The coupling $\pi$ is absolutely continuous with respect to $\Leb^{2d}$ with positive density, $\log\pi$ is $\rmC^2$, and 
\begin{align}
    \|\nabla^2 \log \pi \|_{\mrl^2(\pi)}&:=\left(\int_{\rset^{2d}}\left\| \nabla^2 \log \frac{\rmd\pi}{\rmd \Leb^{2d}}\right\|^2_{\mathrm{op}} \,\rmd \pi \right)^{1/2}\\
    &<+\infty\eqsp.
\end{align}
\end{assumption}
Note that a Lipschitz score function implies \Cref{ass_hessian}.
\paragraph{$\wasserstein_2$ convergence bound without early stopping and constant step-size.}
Under the conditions above, we can establish precise quantitative bounds on the $\wasserstein_2$ distance between the target distribution $\nustar$ and the distribution of the generated samples.\\
\begin{theorem}\label{theo_wasserstein_weak_log_concave}
Let $\{t_k\}_{k=0}^{N_h}$ be a uniform partition of $[0,1]$ with step size $h=1/N_h>0$.
Under \Cref{ass_drift_approx_wass,ass_moment,ass_score,ass_regularity,ass_hessian}, if $\moment[8]{\mu}, \moment[8]{\nustar}\lesssim d^4$, then, denoting by $\nu^{\theta^\star}_1$ the law of $\app_1$, we have 
 \begin{align}
    \label{wass_weak_convergence_bound}
&\wasserstein_2(\nustar,\nu^{\theta^\star}_1)\\
&\lesssim
\mathrm{C}\left(\varepsilon+\sqrt{h}(h^{1/4}+1)
\sqrt{\left(d^2+\norm{\nabla\log \pi}_{\mrl^8(\pi)}^4\right)d}\right)\eqsp.
\end{align}
where \begin{align}\label{def:constant_in_weak_bound}
\mathrm{C} &=\exp\left( \frac{8\sqrt{2}}{\sqrt{\alpha_{\pi}}}\exp\left(\frac{M_\pi}{\alpha_\pi}\right)\|\nabla^2 \log \pi\|_{\mrl^2(\pi)}\right) \eqsp.
\end{align}
\end{theorem}
The bound provided in \Cref{theo_wasserstein_weak_log_concave} naturally decomposes into two additive contributions: one coming from the drift-approximation error, $\varepsilon$, and one from the discretization error, which, consistently with the $\KL$ setting, exhibits a $\bigO(\sqrt{h})$ dependence on the time step and a $\bigO(\sqrt{d^3})$ dependence on the space dimension. Furthermore, an implication of \Cref{theo_wasserstein_weak_log_concave} is that choosing the number of discretization steps
$ N_h = d^3/\varepsilon^2$
guarantees a $\wasserstein_2$ error of order $\mathcal{O}(\varepsilon)$.\\

We can specify our results to the case where  $\pi = \mu \otimes \nustar$.
In this case, it is sufficient that the marginals $\mu$ and $\nustar$ are weakly log-concave in the sense of \cite{conforti2022weak}, and that their scores, together with the corresponding Jacobians, are integrable, to guarantee convergence.

\begin{assumption}\label{ass_indep_coupling}
The marginals $\mu$ and $\nustar$ are absolutely continuous with respect to $\Leb^{d}$ with positive densities, $\log\mu, \log\nustar$ are $\rmC^2$, and
\begin{align}
    &\|\nabla \log \mu \|_{\mrl^8(\mu)}+ \|\nabla\log\nustar\|_{\mrl^8(\nustar)}< +\infty\eqsp,\\
    &\|\nabla^2 \log \mu \|_{\mrl^2(\mu)}+ \|\nabla^2\log\nustar\|_{\mrl^2(\nustar)}< +\infty\eqsp.
\end{align}
Moreover, they satisfy for some $\alpha_\mu, \alpha_{\nustar}>0$, $M_\mu, M_{\nustar}\ge0$ and any $r\ge 0$,
\begin{align}
         &\kappa_{-\log\mu}(r)\ge \alpha_{\mu}-r^{-1}f_{M_{\mu}}(r)\eqsp,\\
          &\kappa_{-\log\nustar}(r)\ge \alpha_{\nustar}-r^{-1}f_{M_{\nustar}}(r)\eqsp.
    \end{align}
\end{assumption}
Under the above conditions, \Cref{ass_score,ass_hessian,ass_regularity} hold:

\begin{lemma}\label{lemma:regularity_par_of_indipendent_coupling}
    Under \Cref{ass_indep_coupling},  \Cref{ass_score,ass_regularity,ass_hessian} hold true for $\pi = \mu \otimes \nustar$. In particular,
    \begin{align}
    &\|\nabla\log\pi\|_{\mrl^8(\pi)}\lesssim \|\nabla\log\mu\|_{\mrl^8(\mu)}+\|\nabla\log\nustar\|_{\mrl^8(\nustar)}\\
    &\|\nabla^2\log\pi\|_{\mrl^2(\pi)}\lesssim \|\nabla^2\log\mu\|_{\mrl^2(\mu)}+\|\nabla^2\log\nustar\|_{\mrl^2(\nustar)}\eqsp,
    \end{align}and $\pi$ is weakly log-concave with parameters 
    \begin{align}
        \alpha_\pi= \min\{\alpha_\mu, \alpha_{\nustar}\}\eqsp, \quad M_\pi=2\max\{M_\mu, M_{\nustar}\}.
    \end{align}
\end{lemma}

We then automatically deduce the following corollary.
\begin{corollary}\label{cor_wass_weak_indep}
Let $\{t_k\}_{k=0}^{N_h}$ be a uniform partition of $[0,1]$ with step size $h=1/N_h>0$. Under \Cref{ass_drift_approx_wass,ass_moment,ass_indep_coupling}, choosing $\pi = \mu \otimes \nustar$ and assuming $\moment[8]{\mu}, \moment[8]{\nustar}\lesssim d^4$, the bound in \Cref{theo_wasserstein_weak_log_concave} holds with $\|\nabla\log\pi\|_{\mrl^8(\pi)}$, $\|\nabla^2\log\pi\|_{\mrl^2(\pi)}$, $\alpha_\pi$ and $M_\pi$ as in
\Cref{lemma:regularity_par_of_indipendent_coupling}.
\end{corollary}
\section{Related works and comparison with existing literature}\label{sec:literature}

The theoretical understanding of DFMs remains limited, despite encouraging empirical results \cite{albergo2023stochastic}. While SGMs
have been extensively studied in terms of both practical performance and
theoretical properties \cite{chen2023improved,conforti2023score,silveri2025beyond,arsenyan2025assessing},
rigorous convergence guarantees for DFM models are relatively recent and
still leave room for improvement.

\paragraph{KL Convergence Bounds.}

\textit{Non-early stopping.} In this setting, DFM has been analyzed in
\cite{gentiloni2024theoretical}, where the authors derived non-asymptotic
KL convergence bounds (Theorem~2). Their analysis accounts for all
practical sources of error, including drift estimation and time
discretization. The bounds require moment conditions on the marginals (Assumption~H1), their
scores (Assumption~H2(i)), and the score of the coupling (Assumption~H2(ii)). The drift estimator is assumed to
approximate the true drift in \(\rml^2\) norm, without any further
smoothness assumptions. The resulting rates scale with the dimension as
\(d^4\). These results provide the first explicit KL guarantees for
non-early-stopped DFM that incorporate all sources of error. In \Cref{theo_no_early}, we improve upon Theorem~2 in \cite{gentiloni2024theoretical} by relaxing the
integrability conditions on the scores of the marginals (Assumption~H2(i))
and by reducing the dimension dependence to \(d^3\). \\

\textit{Early stopping.} The same work \cite{gentiloni2024theoretical} also
considers KL convergence under early stopping. In Theorem~3, the authors
establish explicit bounds on the KL divergence between a smoothed target
distribution and the output of an early-stopped DFM, assuming an independent
coupling \(\pi = \mu \otimes \nustar\), moment conditions on \(\mu\) and \(\nustar\) (Assumption~H1), and moment assumptions
on the score of \(\mu\). The resulting bounds again scale as \(d^4\) and,
to the best of our knowledge, represent the only existing KL guarantees for early-stopped DFM. In \Cref{theo_early}, we consider a more general set of assumptions, which reduces to those of Theorem~3 in \cite{gentiloni2024theoretical} when $\pi$ is chosen as the independent coupling, and we obtain a $\bigO(d^3)$ dimensional scaling. Consequently, \Cref{theo_early} improves upon Theorem~3 of \cite{gentiloni2024theoretical} both in terms of assumptions and dimensional dependence. Moreover, using a tailored step-size schedule and under the same assumptions as \Cref{theo_early}, we obtain in \Cref{theo_faster} accelerated $\KL$ convergence in the early-stopping regime, while maintaining $\mathcal{O}(d^3)$ scaling. From this result, we derive in \Cref{cor:fourier_expo} improved complexity bounds in the Fortet–Mourier metric: $\delta = \bigO(\varepsilon^2/d)$ and $h=\tilde{\bigO}(\varepsilon^2/d^3)$ yield a $\mathcal{W}_{2,\mathrm{FM}}^2$-error of order $O(\varepsilon^2)$. \\

\textit{Two-sided early stopping.}
More recently, \cite{liu2025finite} studied \(\KL\) convergence for
DFMs built on a more general class of
stochastic interpolants, allowing for bridge processes beyond the
classical Brownian bridge. Their analysis assumes finite eighth-order
moments for \(\mu\) and \(\nu^\star\), as well as bounded variation of the time
derivative of the interpolation function (Assumption~4.1), together with
a standard \(\mathrm{L}^2\) approximation error for the learned drift. However, their convergence guarantees (Theorem~4.3) depend crucially on
an early stopping procedure applied not only at the terminal time \(t=1\),
but also at the initial time \(t=0\). Consequently, the $\KL$ bound decomposes into three additive terms.  
The first term, of order $\varepsilon^2$, corresponds to the drift estimation error.  
The second term, $\KL(\mathcal{L}(X^\rmM_{\delta}) \,|\, \hat{\mu})$, with $\delta>0$ and $\hat{\mu}\in\mathcal{P}(\mathbb{R}^d)$ denoting the initialization point of the model \eqref{model}, captures the initialization error: it is strictly positive, analytically intractable, and cannot be controlled in practice.  
The third term accounts for the discretization error and scales as $\mathcal{O}(d^3)$. Notably, the discretization error diverges as the process approaches the prior distribution.
This divergence necessitates truncating
the dynamics away from \(t=0\), which in turn induces the non-vanishing
initialization error. When specializing their bounds to the Brownian bridge
setting (Section~5), the same structural limitations persist: the
initialization error remains non-zero and the discretization error diverges when approaching the prior. As a result, their method forfeits one of the main advantages of
DFMs over SGMs, namely the
ability to generate samples without any intrinsic initialization bias.


\paragraph{Wasserstein-2 Convergence Bounds.}

\textit{Non-early stopping.}
Convergence in the \(2\)-Wasserstein distance has been investigated in
\cite{boffi2025flow} and \cite{anonymous2025pathway}. In
\cite{boffi2025flow}, the authors establish \(2\)-Wasserstein convergence
guarantees for pre-trained diffusion models via Lagrangian and Eulerian
distillation losses, thereby controlling the Wasserstein gap between the
teacher and the student models. However, their analysis requires the
velocity field to satisfy a one-sided Lipschitz condition in space with a
time-dependent Lipschitz constant, and it does not account for
time-discretization errors. Theorem~3.15 in \cite{anonymous2025pathway} provides bounds on the
\(2\)-Wasserstein distance between the generative and target
distributions that explicitly incorporate all sources of error and scale
as \(\mathcal{O}(\sqrt{d})\). It is important to emphasize that these
guarantees are obtained under the assumption of a Gaussian prior
distribution, a setting in which such strong results are to be expected
and which is therefore rather restrictive. Moreover, their analysis
relies on Gaussian tail assumptions on the target distribution
(Assumption~3.7) and on regularity
conditions on the learned velocity field
(Assumption~3.13). In contrast, in \Cref{theo_wasserstein_weak_log_concave} we establish non-asymptotic and fully quantitative convergence guarantees in
\(\wasserstein_2\) distance that simultaneously capture both
drift-approximation and time-discretization errors. These results hold for arbitrary prior distributions and rely only on mild moment and integrability conditions
(\Cref{ass_moment,ass_score,ass_hessian}) on the coupling \(\pi\), its associated score, and the Jacobian of the score, as well as weak convexity assumptions on the log-density of $\pi$ (\Cref{ass_regularity}). Furthermore, our analysis does not require any
\emph{a priori} smoothness assumptions on either the drift or its
estimator; instead, the necessary regularity is derived directly from
\Cref{ass_regularity,ass_hessian}. In \Cref{cor_wass_weak_indep}, we further specialize our bounds to the case of the independent coupling, showing that the same guarantees are obtained by imposing analogous assumptions on the marginals, rather than on the joint coupling (\Cref{ass_indep_coupling}).

\textit{Early stopping.}
The same work \cite{anonymous2025pathway} also establishes
\(\wasserstein_2\)-convergence guarantees under an early stopping
procedure (Theorem~3.19). While these bounds account for all sources of
error and scale as \(\mathcal{O}(\sqrt{d})\), they remain valid only for
Gaussian priors and compactly supported target distributions
(Assumption~3.17). Moreover, the analysis continues to rely on regularity
conditions on the estimator (Assumption~3.13). Therefore, the resulting
\(\wasserstein_2\) bounds exhibit the same structural limitations as in the
non-early-stopped setting.

\section{Methodology}
\paragraph{KL Convergence Bounds.}
Once the standard decomposition of the $\KL$ divergence based on Girsanov’s theorem is applied, and Itô’s formula is used on the process $(\tilde{\beta}_t(X^\mathrm{M}_t))_{t\in [0,T]}$, the problem reduces to obtaining an upper bound on the time-integrated $\mathrm{L}^2$-norm of the mean acceleration field driving the evolution of the mimicking drift, namely $(\partial_t+\foG_t)\fob_t$, where $\foG_t$ denotes the generator of $(\XintM_t)_{t\in[0,1]}$.
This quantity—often referred to as the reciprocal characteristic of the mimicking drift—admits an explicit expansion involving up to three logarithmic derivatives of conditional distributions. To control its integrated norm, we adapt a strategy inspired by \cite{gentiloni2024theoretical}.
The main novelty lies in the treatment of the most delicate contributions, namely the terms involving three logarithmic derivatives. Here, we depart from the approach of \cite{gentiloni2024theoretical} and develop refined bounds. In \cite{gentiloni2024theoretical} these are handled via direct expansion, leading to $d^4$ scaling. Our key idea is to avoid this expansion and instead apply an integration-by-parts argument, which transfers one derivative onto the coupling. This effectively reduces the order of differentiation from three to two, lowering the combinatorial complexity and improving the scaling from $d^4$ to $d^3$.

\paragraph{Wasserstein-2 Convergence Bounds}  Our starting point is the synchronous coupling between the Markovian projection $(X^\mathrm{M}_t)_{t\in[0,1]}$ of the stochastic interpolant and the continuous-time interpolation $(X^{\theta^\star}_t)_{t\in[0,1]}$ of the generative model. The proof then proceeds via a recursive argument in time: for each discretization step $k$, we estimate $\|X^\mathrm{M}_{t_{k+1}}-X^{\theta^\star}_{t_{k+1}}\|_{\mathrm{L}^2}$ in terms of $\left\|X^\mathrm{M}_{t_{k}}-X^{\theta^\star}_{t_{k}}\right\|_{\mathrm{L}^2}$. and propagate the bound along the time grid.Such a recursion requires sufficient regularity of the drift, and establishing this regularity under minimal data assumptions is the main technical challenge. Such a recursion requires sufficient regularity of the drift, and establishing this regularity under minimal data assumptions is the main technical challenge. The key step is to obtain a bound on the Lipschitz constant of the mimicking drift. To this end, we exploit the fact that its Jacobian admits an explicit representation in terms of conditional expectations involving logarithmic derivatives of conditional densities. This representation allows us to leverage the structure of the stochastic interpolant: since it shares the same Brownian bridge, the regularity assumptions imposed on the data distribution can be transported to the associated conditional kernels and their marginals (see \Cref{theo:kernel_regularity}). Building on this, we control the resulting conditional expectations via log-Sobolev and Poincaré inequalities, which yields a quantitative bound on the spatial Lipschitz constant of the mimicking drift (see \Cref{theo:drift_regularity}). Once this regularity estimate is established, the recursive scheme closes and leads to the desired bound. We refer to the proof of \Cref{theo_wasserstein_weak_log_concave} in \Cref{sec:proof_wass} for more details.
\section{Conclusion}
In this work, we conduct a thorough theoretical analysis of a DFM model built upon a $d$-dimensional Brownian bridge. We establish dimension-improved $\KL$ convergence bounds and provide $\wasserstein_2$ convergence guarantees, under mild and standard assumptions. Despite these advances, several avenues remain open for improvement. In particular, it would be highly valuable to relax the integrability requirements on the score functions further; to achieve sharper dependence on the space dimension and to undertake a statistical analysis of DFMs.
\paragraph{Acknowledgments and Disclosure of Funding} The work of Marta Gentiloni-Silveri has been supported by the Paris Ile-de-France Région in the framework of DIM AI4IDF. Alain Durmus has received funding from the Fondation de l’École polytechnique as part of its “Servir la science” campaign. The work of Alain Durmus is supported by the France 2030 program with the reference ANR-25-PEIA0001 (THEOREM project); by Hi! Paris and Agence Nationale de la Recherche (Grant 11-LABX-0047) and by the European Union (ERC-2022-SYG-OCEAN-101071601). Views and opinions expressed are however those of the author(s) only and do not necessarily reflect those of the European Union or the European Research Council Executive Agency. Neither the European Union nor the granting authority can be held responsible for them. 
\paragraph{Impact Statement} This paper presents work whose goal is to advance the field of machine learning. There are many potential societal consequences of our work, none of which we feel must be specifically highlighted here.
\bibliographystyle{icml2026}
    \bibliography{bib.bib}
\newpage
\appendix
\onecolumn

\section{Preliminaries}\label{app:preliminaries}
\paragraph{Extra Notation} Given a matrix $\bfA\in \rset^{n\times s}$, we denote by $\norm{\bfA}_{\mathrm{op}}$ the operator norm of $\bfA$. For $f :\ccint{0,1}\times  \rset^d \to \rset$ regular enough, we denote by $\nabla_x f(t,x), \nabla^2_x f(t,x)$ and $\Delta_x f(t,x)$ respectively the gradient, hessian and laplacian of $f$, defined for $t,x\in [0,1]\times \rset^d$ by $\nabla_x f(t,x):= (\partial_{x_i}f(t,x))_i$, $\nabla^2_x f(t,x):=(\partial_{x_i}\partial_{x_j} f(t,x))_{i,j}$, $\Delta f(t,x):=\sum_{i=1}^{d} \partial^2_{x_i}f(t,x)$, where $\partial_{x_j}$ denotes the partial derivative with respect to the $j$-th variable. For  $F: \ccint{0,1}\times \rset^d\to \rset^d$ regular enough,   we denote by $D_x F$, $\div_x F$ and $\Delta_x F$ respectively, the Jacobian matrix, the divergence and the vectorial laplacian of $F$, defined for $t,x \in \ccint{0,1} \times \rset^d$ by $D_x F(t,x) = (\partial_{x_j}F_i(t,x))_{i,j}$, $\div_x F(t,x):= \sum_{j=1}^{d} \partial_{x_j} F_j(t,x)$, $\Delta_x F(t,x)= (\Delta_x F_1 (t,x),..., \Delta_x F_d (t,x))$.\\

We remark that the stochastic interpolant admits the simple closed form
\begin{equation}\label{interpolant_in_linear_form}
    X^{\rmI}_t \eqlaw (1-t)X^{\rmI}_0+tX^{\rmI}_1+\sqrt{2t(1-t)}\mathrm{Z}\eqsp, \quad \mathrm{Z}\sim \gauss(0, \Id)\eqsp,
  \end{equation}
Denote by  $(p_t^{\rmI})_{t\in\ccint{0,1}}$ the time marginal densities of $(X_t^{\rmI})_{t\in\ccint{0,1}}$ with respect to the Lebesgue measure. Note that, as a consequence of the very definition \eqref{def:stochastic_interpolant} of the stochastic interpolant, they write as
\begin{equation}\label{marginals_interpolant} p^\rmI_t(x)=\int_{\rset^{2d}} p_{t}(x|x_0) p_{1-t}( x_1|x) \tilde{\pi}(\rmd x_0, \rmd x_1)\eqsp, \end{equation}with
\begin{align}\label{def:tilde_pi}
     \tilde{\pi}(\rmd x_0, \rmd x_1)=\frac{ \pi(\rmd x_0, \rmd x_1)}{p_1(x_1|x_0)}\eqsp.
\end{align}

Also, denote by $\mathrm{K}_{t}$ the regular kernel associated with the conditional distribution of $(X^\rmI_0, X^\rmI_1)$ given $X^\rmI_t$, \ie, the map $\mathrm{K}_{t}:\rset^d\times\mathcal{B}(\rset^{2d})\rightarrow [0,1]$ such that
\begin{enumerate}
    \item [(i)] $y\mapsto \mathrm{K}_{t}(y, \mathsf{A})$ is measurable for any $\mathsf{A}\in \mathcal{B}(\rset^{2d})$;
    \item [(ii)] $\mathsf{A}\mapsto \mathrm{K}_{t}(y, \mathsf{A})$ is a probability measure for any $y\in \rset^d$;
    \item [(iii)] almost surely it holds
    \begin{align}
        \mathrm{K}_{t}(X^\rmI_t, \mathsf{A}):=\PE\left[\1_{\mathsf{A}}(X^\rmI_0, X^\rmI_T)|X^\rmI_t\right]\eqsp, \quad \mathsf{A}\in \mathcal{B}(\rset^{2d})\eqsp.
        \end{align}
\end{enumerate}
        It follows again from \eqref{def:stochastic_interpolant} that $ \mathrm{K}_{t}$ admits a transition density with respect to $\Leb^{2d}$ given by
        \begin{align}\label{def:pi_t_x}
        \mathrm{k}_{t}(x_0,x_1|x_t)=p_{t|(0,1)}(x_t|x_0,x_1)(p^\rmI_{t}(x_t))^{-1}\pi(x_0,x_1)\eqsp,
        \end{align}with $p_{t|(0,1)}$ denoting the conditional density of $B_t$ given $(B_0, B_1)$, \ie,
        \begin{align}
            p_{t|(0,1)}(x_0,x_1|x_t)=\frac{p_{1-t}(x_1|x_t)p_t(x_t|x_0)}{p_1(x_1|x_0)}=\exp\left(-\frac{\|x_t-(1-t)x_0-t x_1\|^2}{4t(1-t)}\right)\eqsp.
        \end{align}
Denote also by $ \mathrm{K}_{t}^{\#0}$ and $ \mathrm{K}_{t}^{\#1}$ its first and second marginal respectively and by $\mathrm{k}_{t}^{\#0}:=\rmd  \mathrm{K}_{t}^{\#0}/\rmd \Leb^d,$ and $\mathrm{k}_{t}^{\#1}:=\rmd  \mathrm{K}_{t}^{\#1}/\rmd \Leb^d$ the associated densities.
We observe that, with the introduced notation, for any $t\in [0,1)$ and $y_t\in\rset^d$, we have that
        \begin{align}
            \label{eq:representation_of_drift_as_conditional_expectation_of_linear_function}
            \tilde{\beta}_t(x_t)=\int \frac{x_1-x_t}{1-t}\mathrm{k}_t(\rmd x_0, \rmd x_1|x_t)=\int \frac{x_1-x_t}{1-t}\mathrm{k}^{\#1}_t(\rmd x_1|x_t)\eqsp.
        \end{align}

 \subsection{On the heat kernel}   It is well-known that $(s,x,y)\mapsto p_s(y|x)$ defined in \eqref{def_heat_kernel} is twice continuously differentiable in the space variables $x$ and $y$ and satisfies for $x,y\in \rsetd, s\in (0,1],$  \begin{equation}\label{simmetry_nabla_heat_kernel}
        \nabla_x p_s(y|x)= -\frac{x-y}{2s} p_s(y|x)=-\nabla_y p_s(y|x)\eqsp,
    \end{equation}
    \begin{equation}\label{simmetry_hessian_heat_kernel}
        \nabla^2_x p_s(y|x)= -\frac{1}{2s} p_s(y|x)\Id+ \frac{(x-y)(x-y)^{\transpose}}{4s^2} p_s(y|x)=\nabla^2_y p_s(y|x)\eqsp,
    \end{equation}
\begin{equation}\label{simmetry_delta_heat_kernel}
        \Delta_x p_s(y|x)= -\frac{d}{2s} p_s(y|x)+ \norm{\frac{x-y}{2s}}^2 p_s(y|x)=\Delta_y p_s(y|x)\eqsp.
    \end{equation}Moreover \eqref{def_heat_kernel} satisfies the heat equation, \ie,
    \begin{equation}\label{fp_eq_heat_kernel}
        \partial_s p_s(y|x)=\Delta_x p_s(y|x)\eqsp, \quad s\in (0,1]\eqsp, \eqsp x,y\in \rsetd\eqsp.
    \end{equation}Thus, in particular, $(s,x,y)\mapsto p_s(y|x)$ is continuously differentiable in the time variable $s$.

\subsection{On score's integrability}\label{prel:scor_integrability}
\begin{lemma}\label{lemma:pi_int_implies_marg_int}
    Under \Cref{ass_score}, for any $p\in \{2,4,8\},$ it holds
    \begin{align}\label{eq:bound_score_marginals}
        \|\nabla \log\mu\|_{\mrl^p(\mu)}^p\eqsp,\|\nabla \log\nustar\|_{\mrl^p(\nustar)}^p\lesssim \|\nabla \log\pi\|_{\mrl^p(\pi)}^p\eqsp.
    \end{align}
    Moreover, if we assume also \Cref{ass_moment}, it holds
    \begin{align}\label{eq:bound_score_tilde_pi}
        \|\nabla \log\tilde{\pi}\|_{\mrl^p(\pi)}\le \|\nabla \log\pi\|_{\mrl^p(\pi)}+\sqrt[p]{\moment[p]{\mu}}+\sqrt[p]{\moment[p]{\nustar}}\eqsp.
    \end{align}
\end{lemma}
\begin{proof}[Proof of \Cref{lemma:pi_int_implies_marg_int}:]
    We start with \eqref{eq:bound_score_marginals}. We only deal with $\mu$, as the argument for $\nustar$ is the very same. We have that
    \begin{align}
        \nabla \log\mu(x_0)&= \frac{\nabla \mu(x_0)}{\mu(x_0)}= \frac{ \int \nabla_{x_0} \pi(x_0,x_1)\rmd x_1}
       {\int \pi(x_0,x_1)\rmd x_1}
= \frac{\int\pi(x_0,x_1)\nabla_{x_0}\log\pi(x_0,x_1)\rmd x_1}{\int \pi(x_0,x_1)\rmd x_1}\\
&=\PE\left[\nabla_{x_0}\log\pi(X^\rmI_0,X^\rmI_1)|X^\rmI_0=x_0\right]\eqsp.
    \end{align}
Jensen inequality yields that
\begin{align}
    \|\nabla \log\mu(x_0)\|^p\le \PE\left[\left\|\nabla_{x_0}\log\pi(X^\rmI_0,X^\rmI_1)\right\|^p|X^\rmI_0=x_0\right]\eqsp.
\end{align}
Taking expectation and using the properties of conditional expectation, we get that
\begin{align}
    \|\nabla \log\mu\|_{\mrl^p(\mu)}^p&=\PE\left[\|\nabla \log\mu(X^\rmI_0)\|^p\right]\le \PE\left[ \PE\left[\left\|\nabla_{x_0}\log\pi(X^\rmI_0,X^\rmI_1)\right\|^p|X^\rmI_0\right]\right]\\
    &= \PE\left[ \left\|\nabla_{x_0}\log\pi(X^\rmI_0,X^\rmI_1)\right\|^p\right]=\|\nabla_{x_0} \log\pi\|_{\mrl^p(\pi)}^p\le \|\nabla \log\pi\|_{\mrl^p(\pi)}^p\eqsp.
\end{align}
The bound in \eqref{eq:bound_score_tilde_pi} is a direct consequence of the definitions \eqref{def:tilde_pi} of $\tilde{\pi}$ and \eqref{def_heat_kernel} of $p_1(x_1|x_0)$. 
\end{proof}

\subsection{On the Stochastic Interpolant}
We collect here several results on Stochastic Interpolants that will be used later. \Cref{lemma_on_moments_bounded} reports the result of Lemma 1 in \cite{gentiloni2024theoretical} and is a direct consequence of \eqref{interpolant_in_linear_form}. \Cref{lemma:lemma2_back_in_dfm,lemma:lemma2_for_in_dfm} are taken from Lemma 2 in \cite{gentiloni2024theoretical}, while \Cref{lemma:lemma3_back_in_dfm,lemma:lemma3_for_in_dfm} follow from Lemma 3 in \cite{gentiloni2024theoretical} and our \Cref{lemma:pi_int_implies_marg_int}. We refer to \cite{gentiloni2024theoretical} for proofs.
\begin{lemma}\label{lemma_on_moments_bounded}
    For any $p\ge 1$, they hold
    \begin{equation}
         \PE[\norm{X^{\rmI}_s-X^{\rmI}_0}^{2p}] \lesssim s^{2p} \moment[2p]{\mu}+s^{2p}\moment[2p]{\nustar}+d^{p}s^p(1-s)^p\eqsp, 
    \end{equation}and
    \begin{equation}
    \PE[\norm{X^{\rmI}_1-X^{\rmI}_s}^{2p}]\rmd s \lesssim (1-s)^{2p}\moment[2p]{\mu}+(1-s)^{2p}\moment[2p]{\nustar}+d^{p}s^p(1-s)^p\eqsp.
    \end{equation}
\end{lemma}
\begin{proposition}\label{lemma:lemma2_back_in_dfm}
    The time reversal of the stochastic interpolant $((X^\rmI_t)_{t\in [0,1]})^\rmR$ solves weakly \begin{equation}\label{def:X_backward}
        \rmd \baX_t= 2\babi_t(\baX_0, \baX_t)\rmd t +  \sqrt{2}\rmd \baB_t\eqsp, \quad t\in [0,1]\eqsp, \quad \baX_0\sim \nustar\eqsp.
    \end{equation}
    with $(\baB_t)_{t\in [0,1]}$ $d$-dimensional Brownian motion independent of $\baX_0$ and $(t,x)\mapsto \babi_t(x)$ as in \cite{gentiloni2024theoretical}, Lemma 2. In particular, they hold
\begin{equation}\label{eq:law_eq_XI_Xba}
     (X^\rmI_t)_{t\in [0,1]} \eqlaw ((\baX_t)_{t\in [0,1]})^{\rmR}\eqsp,
\end{equation}and, for any $u\in [0,1]$ and $t\in [0, 1-u]$,
\begin{equation}\label{eq:XI_toXba}
    X^{\rmI}_{1-(t+u)}-X^{\rmI}_{1-u}\eqlaw \overleftarrow{f}_u^t+\overleftarrow{g}_u^t\eqsp,\quad
       \overleftarrow{f}_u^t=2\int_{u}^{t+u} \babi_r (\baX_0, \baX_r)\rmd r\eqsp, \quad \overleftarrow{g}_u^t=\baB_{t+u}-\baB_{u}\eqsp.
\end{equation}Moreover, $(\baX_t)_{t\in [0,1]}$ is $(\overleftarrow{\mathcal{F}}_t)_{t\in [0,1]}$-adapted, with $\overleftarrow{\mathcal{F}}_t=\sigma(\baX_0, (\baB_u)_{u\le t})$.
\end{proposition}
\begin{lemma}\label{lemma:lemma3_back_in_dfm}
Assume \Cref{ass_moment,ass_score}. For any $u\in [0,1]$, $t\in [0, 1-u]$, $p\in \{2,4,8\}$, we have that
    \begin{equation}
        \PE\Big[\norm{\overleftarrow{f}_u^t}^{p}\Big]\lesssim t^p \left( \norm{\nabla\log \pi}_{\mrl^p(\pi)}^p +\moment[p]{\mu}+\moment[p]{\nustar}\right)\eqsp.
    \end{equation}
\end{lemma}
\begin{proposition}\label{lemma:lemma2_for_in_dfm}
    The stochastic interpolant $(X^\rmI_t)_{t\in [0,1]}$ solves weakly
    \begin{equation}\label{def:X_forward}
        \rmd \foX_t= 2\fobi_t(\foX_0, \foX_t)\rmd t +  \sqrt{2}\rmd \foB_t\eqsp, \quad t\in [0,1]\eqsp, \quad \foX_0\sim \mu\eqsp.
    \end{equation}
    with $(\foB_t)_{t\in [0,1]}$ $d$-dimensional Brownian motion independent of $\foX_0$ and $(t,x)\mapsto \fobi_t(x)$ as in \cite{gentiloni2024theoretical}, Lemma 2. In particular, they hold
\begin{align}
     (X^\rmI_t)_{t\in [0,1]} \eqlaw (\foX_t)_{t\in [0,1]}\eqsp,
\end{align}and for any $u\in [0,1]$ and $t\in [0, 1-u]$,
\begin{equation}
    X^{\rmI}_{t+u}-X^{\rmI}_{u}\eqlaw \overrightarrow{f}_u^t+\overrightarrow{g}_u^t\eqsp, \quad \overrightarrow{f}_u^t=2\int_{u}^{t+u} \fobi_r (\foX_0, \foX_r)\rmd r\eqsp, \quad \overrightarrow{g}_u^t=\foB_{t+u}-\foB_{u}\eqsp.
\end{equation}Moreover, $(\foX_t)_{t\in [0,1]}$ is $(\overrightarrow{\mathcal{F}}_t)_{t\in [0,1]}$-adapted, with $\overrightarrow{\mathcal{F}}_t=\sigma(\foX_0, (\foB_u)_{u\le t})$.
\end{proposition}
\begin{lemma}\label{lemma:lemma3_for_in_dfm}
Assume \Cref{ass_moment,ass_score}. For any $u\in [0,1]$, $t\in [0, 1-u]$, $p\in \{2,4,8\}$, we have that
\begin{equation}
        \PE\Big[\norm{\overrightarrow{f}_u^t}^{p}\Big]\lesssim t^p  \left(\norm{\nabla\log \pi}_{\mrl^p(\pi)}^p +\moment[p]{\mu}+\moment[p]{\nustar}\right)\eqsp,
    \end{equation}
\end{lemma}

\subsection{On the conditional Markov kernel}

\subsubsection{Strongly log-concave case}
Strongly log-concave couplings satisfy \Cref{ass_regularity}. More specifically, let
\begin{assumption}[Strong log-concavity of the coupling]\label{ass_strong_regularity} $\pi\in\mcp(\rset^{2d})$ is absolutely continuous with respect to $\Leb^{2d}$, $\log\pi$ is $\rmC^2$, and satisfies \begin{align} \nabla^2(-\log\pi)\succeq \alpha_{\pi}\Id\eqsp, \end{align}  \end{assumption}
then $\pi$ satisfies \Cref{ass_regularity} with $M_\pi=0$. For sake of clarity, we first consider the case where $\pi\in\Pi(\mu,\nustar)$ satisfies \Cref{ass_strong_regularity}.
\begin{lemma}\label{theo:strong_kernel_regularity}
    Assume that $\pi\in\Pi(\mu,\nustar)$ satisfies \Cref{ass_strong_regularity}. Fix $t\in (0,1)$ and $y_t\in \rsetd$. Then, $\mathrm{K}_t(y_t, \cdot)$ satisfies \Cref{ass_strong_regularity} with the same parameter $\alpha_{\pi}$. Moreover, $\mathrm{K}^{\#0}_t(y_t, \cdot)$ and $\mathrm{K}^{\#1}_t(y_t, \cdot)$ satisfy \Cref{ass_strong_regularity} as well with better parameters. Specifically,
    \begin{align}\label{eq:strong_regularity_conditioned_kernel}
   \nabla^2(-\log\mathrm{k}^{\#0}_t(y_t|\cdot))\succeq  \left(\alpha_{\pi} +\frac{1-t}{2t} \right) \Id\eqsp, \quad \nabla^2(-\log\mathrm{k}^{\#1}_t(y_t|\cdot))\succeq  \left(\alpha_{\pi} +\frac{t}{2(1-t)} \right) \Id\eqsp.
\end{align}

\end{lemma}
\begin{proof}[Proof of \Cref{theo:strong_kernel_regularity}:]
     Fix $t\in [0,1)$ and $y_t\in \rsetd$ and denote by
\begin{align}
    V(y_0,y_1):= -\log \pi(y_0,y_1) \eqsp.
\end{align}Then, because of \eqref{def:pi_t_x}, it holds
\begin{align}
    \mathrm{k}_{t}(y_0,y_1|y_t)\propto \exp\left(-\frac{\|y_t-(1-t)y_0-t y_1\|^2}{4t(1-t)}-V(y_0,y_1)\right)\eqsp,
\end{align}hence
\begin{align}\label{eq:log_conditional_density}
    -\log \mathrm{k}_{t}(y_0,y_1|y_t)= \frac{\|y_t-(1-t)y_0-t y_1\|^2}{4t(1-t)}+V(y_0,y_1)\eqsp,
\end{align}and
\begin{align}\label{eq:hessians_conditional_marginals}
   \nabla^2_{y_0} (-\log \mathrm{k}_{t}(y_0,y_1|y_t))= \frac{1-t}{2t}+\nabla^2_{y_0}V(y_0,y_1)\eqsp,\eqsp  \nabla^2_{y_1} (-\log \mathrm{k}_{t}(y_0,y_1|y_t))= \frac{t}{2(1-t)}+\nabla^2_{y_1}V(y_0,y_1)\eqsp.
\end{align}At this pont, \eqref{eq:strong_regularity_conditioned_kernel} follows directly from \Cref{ass_strong_regularity}.

\end{proof}
\subsubsection{Weakly log-concave case}
We now consider the general case where $\pi\in\Pi(\mu,\nustar)$ satisfies \Cref{ass_regularity}.
 \begin{lemma}\label{theo:kernel_regularity}
     Assume that $\pi\in \Pi(\mu,\nustar)$ satisfies \Cref{ass_regularity}. Fix $t\in (0,1)$ and $y_t\in \rsetd$. Then, $\mathrm{K}_t(y_t, \cdot)$ satisfies \Cref{ass_regularity} with the same parameters $\alpha_{\pi}, M_{\pi}$. Moreover, $\mathrm{K}^{\#0}_t(y_t, \cdot)$ and $\mathrm{K}^{\#1}_t(y_t, \cdot)$ satisfy \Cref{ass_regularity} with better parameters. Specifically,
   \begin{align}\label{def:param_weak_conv_improved}
       k_{-\log\mathrm{k}^{\#0}_t}(r)\ge \left(\alpha_\pi+\frac{1-t}{2t}\right)-r^{-1}f_{M_\pi}(r)\eqsp, \quad k_{-\log\mathrm{k}^{\#1}_t}(r)\ge \left(\alpha_\pi+\frac{t}{2(1-t)}\right)-r^{-1}f_{M_\pi}(r)\eqsp.
   \end{align}



\end{lemma}

 \begin{proof}[Proof of \Cref{theo:kernel_regularity}:] Fix $t\in (0,1)$ and $y_t\in \rsetd$. Proceed as in the proof of \Cref{theo:strong_kernel_regularity} to get \eqref{eq:log_conditional_density}. Using \Cref{ass_regularity} and \eqref{eq:log_conditional_density}, we get that $\mathrm{K}_t(y_t, \cdot)$ satisfies \Cref{ass_regularity} with the same parameters $\alpha_{\pi}, M_{\pi}$ and that 
the marginals of $\mathrm{K}_t(y_t, \cdot)$ satisfy \Cref{ass_regularity} with parameters given by \eqref{def:param_weak_conv_improved}.
 
 \end{proof}

\subsection{On the mimicking drift's regularity}

\subsubsection{Strongly log concave case}
We preface this section with an auxiliary technical lemma.
\begin{lemma}\label{lemma:operator_corm_block_of_matrix}
Let 
\begin{align}
    Y =
\begin{pmatrix}
Y^{(1)} \\[2mm]
Y^{(2)}
\end{pmatrix}\eqsp, 
\quad Y^{(1)},  Y^{(2)} \in \rsetd\eqsp,
\end{align}be a random vector with finite second order moments. Then
\begin{align}
    \left\|\mathrm{Cov}(Y^{(1)},  Y^{(2)})\right\|_{\mathrm{op}}\le \sqrt{ \left\|\mathrm{Cov}(Y^{(1)})\right\|_{\mathrm{op}} \left\|\mathrm{Cov}(  Y^{(2)})\right\|_{\mathrm{op}}}\eqsp.
\end{align}
\end{lemma}
\begin{proof}[Proof of \Cref{lemma:operator_corm_block_of_matrix}:]
Let $\Sigma:=\mathrm{Cov}(Y)$. Then $\Sigma$ can be written in block form as
\begin{align}
    \Sigma =
\begin{pmatrix}
\Sigma_{11} & \Sigma_{12} \\[1mm]
\Sigma_{21} & \Sigma_{22}
\end{pmatrix}\eqsp,
\end{align}
where
\begin{align}
    \Sigma_{11} = \mathrm{Cov}(Y^{(1)})\eqsp, \quad
\Sigma_{22} = \mathrm{Cov}(Y^{(2)})\eqsp, \quad
\Sigma_{12} = \mathrm{Cov}(Y^{(1)}, Y^{(2)})\eqsp, \quad
\Sigma_{21} = \Sigma_{12}^T\eqsp.
\end{align}
Recall that, for any matrix $\mathrm{A}$ or order $d$, we have
\begin{align}\label{def_op_norm}
    \|\mathrm{A}\|_{\mathrm{op}}=\sup_{u,v\in\rsetd, \|u\|,\|v\|=1} |u^\top \mathrm{A} v|\eqsp.
\end{align}
Let $u,v \in \rsetd$ be arbitrary unit vectors. Being $\Sigma$ positive semi-definite, for any $t\in \rset$, we have that
\begin{align}\label{eq_in_lemma_discriminant}
    \begin{pmatrix} u \\ tv \end{pmatrix}^T
\Sigma
\begin{pmatrix} u \\ tv \end{pmatrix} =
u^T \Sigma_{11} u + \left(2 u^T \Sigma_{12} v\right)t +\left( v^T \Sigma_{22} v\right)t^2 \ge 0\eqsp.
\end{align}
This implies that the discriminant of the above second order polynomial in $t$ is always non-positive, \ie,
\begin{align}
    \big(2\, u^T \Sigma_{12} v\big)^2 - 4 \,(u^T \Sigma_{11} u) (v^T \Sigma_{22} v)\le 0\eqsp,
\end{align}
or equivalently,
\begin{align}
    (u^T \Sigma_{12} v)^2 \le (u^T \Sigma_{11} u)( v^T \Sigma_{22} v )\eqsp.
\end{align}
It follows from \eqref{def_op_norm} that
\begin{align}
   |u^T \Sigma_{12} v| \le \sqrt{\|\Sigma_{11}\|_{\mathrm{op}} \|\Sigma_{22}\|_{\mathrm{op}}}\eqsp. 
\end{align}
The thesis then follows from the arbitrary of the unit vectors $u,v\in\rsetd.$
\end{proof}
\begin{proposition}\label{theo:strong_drift_regularity}
Assume that $\pi\in \Pi(\mu,\nustar)$ satisfies \Cref{ass_strong_regularity,ass_hessian}. Then, for any
$s\in (0,1)$ and $x,y\in\rsetd$, it holds
    \begin{align}\label{eq:strong_lipschitz_mimicking_drift}
        \|\tilde{\beta}_s(x)-\tilde{\beta}_s(y)\|\le \frac{4\sqrt{2}}{\sqrt{\alpha_{\pi}}} \|\nabla^2 \log \pi\|_{\mrl^2(\pi)}\frac{1}{\sqrt{s(1-s)}}\|x-y\|\eqsp.
    \end{align}
\end{proposition}
\begin{proof}[Proof of \Cref{theo:strong_drift_regularity}:] Few computations lead to
\begin{equation} \label{jacobian_mimicking_drift}
        \begin{split}
            & D_x \fob_s(x)\\
            &= 2 \frac{\int_{\rset^{2d}} \nabla_x p_{1-s}(x_1|x)(\nabla_x p_s(x|x_0))^{\transpose} \tilde{\pi}(x_0,x_1) \rmd x_0 \rmd x_1}{p^{\rmI}_s(x)}\\
            &+ 2 \frac{\int_{\rset^{2d}} p_{s}(x|x_0) \nabla^2_x p_{1-s}(x_1|x) \tilde{\pi}(x_0,x_1) \rmd x_0 \rmd x_1}{p^{\rmI}_s(x)} \\
             &-2  \Big(\frac{\int_{\rset^{2d}} p_{s}(x|x_0) \nabla_x  p_{1-s}(x_1|x) \tilde{\pi}(x_0,x_1) \rmd x_0 \rmd x_1}{p^{\rmI}_s(x)}\Big)\Big(\frac{ \int_{\rset^{2d}} \nabla_x p_{s}(x|x_0) p_{1-s}(x_1|x) \tilde{\pi}(x_0,x_1) \rmd x_0 \rmd x_1}{p^{\rmI}_s(x)}\Big)^{\transpose} \\
             &- 2\Big(\frac{\int_{\rset^{2d}} p_{s}(x|x_0) \nabla_x p_{1-s}(x_1|x) \tilde{\pi}(x_0,x_1) \rmd x_0 \rmd x_1}{p^{\rmI}_s(x)}\Big) \Big(\frac{\int_{\rset^{2d}} p_{s}(x|x_0) \nabla_x p_{1-s}(x_1|x) \tilde{\pi}(x_0,x_1) \rmd x_0 \rmd x_1}{p^{\rmI}_s(x)}\Big)^{\transpose}\eqsp.
        \end{split}
    \end{equation}

Using \eqref{simmetry_nabla_heat_kernel} and integration by parts formula, we get
\begin{align}
     & D_x \fob_s(x)\\
            &=  \frac{\int_{\rset^{2d}} \frac{\nabla_{x_1}\tilde{\pi}}{\tilde{\pi}}(x_0,x_1)(\frac{x_0-x}{s})^{\transpose} p_{s}(x|x_0)p_{1-s}(x_1|x)\tilde{\pi}(x_0,x_1) \rmd x_0 \rmd x_1}{p^{\rmI}_s(x)}\\
            &+ \frac{\int_{\rset^{2d}} \frac{x_1-x}{1-s}(\frac{\nabla_{x_1}\tilde{\pi}}{\tilde{\pi}}(x_0,x_1))^{\transpose} p_{s}(x|x_0)p_{1-s}(x_1|x) \tilde{\pi}(x_0,x_1) \rmd x_0 \rmd x_1}{p^{\rmI}_s(x)}\\
             &-  \Big(\frac{\int_{\rset^{2d}} \frac{\nabla_{x_1}\tilde{\pi}}{\tilde{\pi}}(x_0,x_1)p_{s}(x|x_0)  p_{1-s}(x_1|x) \tilde{\pi}(x_0,x_1) \rmd x_0 \rmd x_1}{p^{\rmI}_s(x)}\Big)\\
             &\quad \cdot\Big(\frac{ \int_{\rset^{2d}}\frac{x_0-x}{s} p_{s}(x|x_0) p_{1-s}(x_1|x) \tilde{\pi}(x_0,x_1) \rmd x_0 \rmd x_1}{p^{\rmI}_s(x)}\Big)^{\transpose} \\
             &- \Big(\frac{\int_{\rset^{2d}} \frac{x_1-x}{1-s}p_{s}(x|x_0)  p_{1-s}(x_1|x) \tilde{\pi}(x_0,x_1) \rmd x_0 \rmd x_1}{p^{\rmI}_s(x)}\Big) \\
             &\quad \cdot\Big(\frac{\int_{\rset^{2d}} \frac{\nabla_{x_1}\tilde{\pi}}{\tilde{\pi}}(x_0,x_1)p_{s}(x|x_0) p_{1-s}(x_1|x) \tilde{\pi}(x_0,x_1) \rmd x_0 \rmd x_1}{p^{\rmI}_s(x)}\Big)^{\transpose}\\
             &\label{eq:jacobian_as_covariance}=\text{Cov}_{(X^{\rmI}_0,X^\rmI_1)\sim \mathrm{K}_s(x,\cdot)}\left(\frac{X^\rmI_0-x}{s}, \frac{\nabla_{x_1}\tilde{\pi}}{\tilde{\pi}}(X^{\rmI}_0,X^\rmI_1)\right)+\text{Cov}_{(X^{\rmI}_0,X^\rmI_1)\sim \mathrm{K}_s(x,\cdot)}\left(\frac{X^\rmI_1-x}{1-s}, \frac{\nabla_{x_1}\tilde{\pi}}{\tilde{\pi}}(X^{\rmI}_0,X^\rmI_1)\right)\eqsp.
\end{align}
Therefore, using \Cref{lemma:operator_corm_block_of_matrix}, \Cref{ass_strong_regularity,ass_hessian}, \Cref{theo:strong_kernel_regularity}, Brascamp-Lieb inequality (see Lemma 2 in \cite{chewi2023entropic}) and \Cref{lemma:pi_int_implies_marg_int}, we get
\begin{align}
      \| D_x \fob_s(x)\|_{\mathrm{op}}&\le \left(\sqrt{\left\|\text{Cov}_{X^\rmI_0\sim \mathrm{K}^{\# 0}_s(x,\cdot)}\left(\frac{X^\rmI_0-x}{s}\right)\right\|_{\mathrm{op}}}+\sqrt{\left\|\text{Cov}_{X^\rmI_1\sim \mathrm{K}^{\# 1}_s(x,\cdot)}\left(\frac{X^\rmI_1-x}{1-s}\right)\right\|_{\mathrm{op}}}\right)\\
             &\quad \cdot\sqrt{\left\|\text{Cov}_{(X^{\rmI}_0,X^\rmI_1)\sim \mathrm{K}_s(x,\cdot)}\left(\nabla\log\tilde{\pi}(X^{\rmI}_0,X^\rmI_1)\right)\right\|_{\mathrm{op}}}\\
    &\preceq \left(\frac{1}{s}\sqrt{\frac{2s}{2\alpha_{\pi}s+ 1-s} }+ \frac{1}{1-s}\sqrt{\frac{2(1-s)}{s+2\alpha_\pi(1-s)}}\right)\sqrt{\frac{\|\nabla \log \tilde{\pi}\|_{\mrl^2(\pi)}^2}{\alpha_{\pi}}}\\
    &\le \frac{4\sqrt{2}}{\sqrt{\alpha_{\pi}}} \|\nabla^2 \log \pi\|_{\mrl^2(\pi)}\frac{1}{\sqrt{s(1-s)}}\eqsp.
\end{align}
Hence, we have that
\begin{align}
    \| D_x \fob_s(x)\|_{\mathrm{op},\infty}&\le \frac{4\sqrt{2}}{\sqrt{\alpha_{\pi}}} \|\nabla^2 \log \pi\|_{\mrl^2(\pi)}\frac{1}{\sqrt{s(1-s)}}\eqsp.
\end{align}The thesis follows directly.
\end{proof}
\subsubsection{Weakly log concave case}

\begin{proposition}\label{theo:drift_regularity}
    Assume that $\pi\in \Pi(\mu,\nustar)$ satisfies \Cref{ass_regularity,ass_hessian}. Then, for any
$s\in (0,1)$ and $x,y\in\rsetd$, it holds 
    \begin{align}\label{eq:extremes_lipschitz_mimicking_drift}
        \|\tilde{\beta}_s(x)-\tilde{\beta}_s(y)\|\le  \frac{4\sqrt{2}}{\sqrt{\alpha_{\pi}}}\exp\left(\frac{M_\pi}{\alpha_\pi}\right) \|\nabla^2 \log \pi\|_{\mrl^2(\pi)}\frac{1}{\sqrt{s(1-s)}}\eqsp.
    \end{align}
\end{proposition}
\begin{proof}[Proof of \Cref{theo:drift_regularity}:] Proceed as in the proof of \Cref{theo:strong_drift_regularity}, to get \eqref{eq:jacobian_as_covariance} and therefore
\begin{align}\label{eq:upper_bound_jacobian}
     \|D_x \fob_s(x)\|_{\mathrm{op}}&\le \left(\sqrt{\left\|\text{Cov}_{X^\rmI_0\sim \mathrm{K}^{\# 0}_s(x,\cdot)}\left(\frac{X^\rmI_0-x}{s}\right)\right\|_{\mathrm{op}}}+\sqrt{\left\|\text{Cov}_{X^\rmI_1\sim \mathrm{K}^{\# 1}_s(x,\cdot)}\left(\frac{X^\rmI_1-x}{1-s}\right)\right\|_{\mathrm{op}}}\right)\\
             &\quad \cdot\sqrt{\left\|\text{Cov}_{(X^{\rmI}_0,X^\rmI_1)\sim \mathrm{K}_s(x,\cdot)}\left(\nabla\log\tilde{\pi}(X^{\rmI}_0,X^\rmI_1)\right)\right\|_{\mathrm{op}}}\eqsp.
\end{align}
For any $\alpha>0, M\ge 0$, denote by
\begin{align}\label{eq:poincare_constant}
\xi(\alpha,M):=\frac{\alpha}{\exp\left(M/\alpha\right)}\eqsp.
\end{align}Also, fix $s\in (0,1)$ and $y_s\in \rsetd$. The byproduct of \Cref{theo:kernel_regularity} and Theorem 5.7 in \cite{conforti2023projected} implies that $\mathrm{K}_s(y_s, \cdot)$, $\mathrm{K}^{\#0}_s(y_s, \cdot)$ and $\mathrm{K}^{\#1}_s(y_s, \cdot)$ satisfy log-Sobolev inequality with parameters $\xi^{-1}(\alpha_{\pi},M_{\pi})$, $\xi^{-1}(\alpha_{\pi}+(1-s)/(2s),M_{\pi})$ and $\xi^{-1}(\alpha_{\pi}+s/(2(1-s)),M_{\pi})$ respectively, with $\xi(\alpha,M)$ for $\alpha>0, M\ge 0$ defined in \eqref{eq:poincare_constant}.
Consequently, as shown in \cite{villani2008optimal}, $\mathrm{K}_s(y_t, \cdot)$, $\mathrm{K}^{\#0}_s(y_s, \cdot)$ and $\mathrm{K}^{\#1}_s(y_s, \cdot)$ satisfy Poincaré inequality with constants $2\xi(\alpha_{\pi},M_{\pi})$, $2\xi^{-1}(\alpha_{\pi}+(1-s)/(2s),M_{\pi})$ and $2\xi^{-1}(\alpha_{\pi}+s/(2(1-s)),M_{\pi})$ respectively. We therefore get that

\begin{align}
     &\|D_x \fob_s(x) \|_{\mathrm{op}}\\
    &\le \left( \frac{1}{s}\sqrt{2\xi^{-1}\left(\alpha_{\pi}+\frac{1-s}{2s},M_{\pi}\right)}+ \frac{1}{1-s}\sqrt{2\xi^{-1}\left(\alpha_{\pi}+\frac{s}{2(1-s)},M_{\pi}\right)}\right) \|\nabla^2 \log \pi\|_{\mrl^2(\pi)}\sqrt{2\xi^{-1}(\alpha_{\pi},M_{\pi})}\\
    &=\left(\frac{2\sqrt{2}}{\sqrt{s}}\frac{1}{\sqrt{2s\alpha_\pi+1-s}}\exp\left(\frac{M_\pi}{\alpha_\pi+(1-s)/(2s)}\right)+\frac{2\sqrt{2}}{\sqrt{1-s}}\frac{1}{\sqrt{2(1-s)\alpha_\pi+s}}\exp\left(\frac{M_\pi}{\alpha_\pi+s/(2(1-s))}\right)
    \right)\|\nabla^2 \log \pi\|_{\mrl^2(\pi)}\\
    &\le \frac{4\sqrt{2}}{\sqrt{\alpha_{\pi}}}\exp\left(\frac{M_\pi}{\alpha_\pi}\right) \|\nabla^2 \log \pi\|_{\mrl^2(\pi)}\frac{1}{\sqrt{s(1-s)}}\eqsp.
\end{align}
\end{proof}

\section{Convergence Bounds in Kullback–Leibler Divergence}\label{app:proof_kl_bounds}
\subsection{Non-early-stopping regime with constant step-size}
\begin{proof}[Proof of \Cref{theo_no_early}:]
  \label{sec:proof-crefth}
  For any $t\in [0,1]$, we denote by $\nustar_t=\text{Law}(\XintM_t)$ and $\nu^{\theta^{\star}}_t=\text{Law}(\app_t).$ Also, we denote by $\foG$ the generator  of $(\XintM_t)_{t\in [0,1]}$, which is defined for any $t\in [0,1]$ and $\rho \in C^2(\rsetd)$ as
\begin{equation}\foG_t \rho :=\langle\nabla_x \rho ,\fob_t\rangle + \Delta_x \rho \eqsp.
\end{equation}
 We fix $0<\epsilon_1<\min\{h, 1/2\}$.
  First, using the data processing inequality (see, $e.g.$, Lemma 1.6 in \cite{nutz2021introduction}), the standard decomposition of the KL divergence \cite{chen2022sampling,chen2023improved,conforti2023score} based on Girsanov theorem, triangle inequality and \Cref{ass_drift_approx}, we bound the KL divergence between $\nustar_{1-\epsilon_1}$ and $\nu^{\theta^{\star}}_{1-\epsilon_1}$ as follows

        \begin{align}\label{Decomposition_KL_Girsanov}
            &\KL(\nustar_{1-\epsilon_1}|\nu^{\theta^{\star}}_{1-\epsilon_1})\\
            & \le \KL(\text{Law}((\XintM_t)_{t\in [0, 1-\epsilon_1]}) | \mathrm{Law}((\app_t)_{t\in [0,1-\epsilon_1]}))\\
            &\lesssim \sum_{k=0}^{N-2}\int_{t_k}^{t_{k+1}} \PE \Big[\norm{ s_{\theta^\star}(t_k, \XintM_{t_k})- \fob_{t}(\XintM_{t})}^2\Big]\rmd t + \int_{1-h}^{1-\epsilon_1} \PE \Big[\norm{ s_{\theta^\star}(1-h, \XintM_{1-h})- \fob_{t}(\XintM_{t})}^2\Big]\rmd t\\
    &\lesssim \varepsilon^2 +\sum_{k=0}^{N-2}\int_{t_k}^{t_{k+1}} \PE \Big[\norm{ \fob_{t_k}(\XintM_{t_k})- \fob_{t}(\XintM_{t})}^2\Big]\rmd t + \int_{1-h}^{1-\epsilon_1} \PE \Big[\norm{ \fob_{1-h}(\XintM_{1-h})- \fob_{t}(\XintM_{t})}^2\Big]\rmd t\eqsp.
    \end{align}
Second, we aim at bounding the RHS of \eqref{Decomposition_KL_Girsanov} uniformly in $\epsilon_1$. Indeed, if we assume to be able to bound it with a constant $A$ independent of $\epsilon_1$, then, using the weak convergence of $\XintM_{1-\epsilon_1}$ to $X^\rmI_1$ (whose law is given by $\nustar$) as $\epsilon_1 \to 0$, the continuity of $(\app_t)_{t\in [0,1]}$, hence the weak convergence of $\app_{1-\epsilon_1}$ to $\app_1$ (whose law is given by $\nu^{\theta^\star}_1$) as $\epsilon_1 \to 0$, and the lower semi-continuity of the $\KL$-divergence with respect to the weak convergence (see, $e.g.$, Theorem 19 in \cite{van2014renyi}), we will get
\begin{equation}\label{lsc_of_kl}
        \begin{split}
            \KL(\nustar|\nu^{\theta^{\star}}_{1}) \le \liminf_{\epsilon_1\to 0} \KL(\nustar_{1-\epsilon_1}|\nu^{\theta^{\star}}_{1-\epsilon_1})\lesssim \liminf_{\epsilon_1\to 0} A=A \eqsp.
        \end{split}
    \end{equation}
Let us therefore bound the RHS of \eqref{Decomposition_KL_Girsanov}. We will do so by using stochastic calculus tools, and, more precisely, Ito's formula. This formula yields
    \begin{equation}\label{ito_dec_drift}
            \rmd \fob_t(\XintM_t)= (\partial_t +\foG_t)\fob_t(\XintM_t) \rmd t + \sqrt{2} D_x \fob_t(\XintM_t) \rmd W_t\eqsp,\quad t\in [0,1-\epsilon_1]\eqsp.
    \end{equation}So, applying Young inequality and Ito's isometry, we have that, for any $k=0,\cdots,N-1$
    \begin{equation}\label{bound_via_ito_of_difference_of_b_t}
    \begin{split}
        &\PE \Big[\norm{ \fob_{t_k}(\XintM_{t_k})- \fob_{t}(\XintM_{t})}^2\Big]\\
        &=\PE \Bigg[\norm{\int_{t_k}^t(\partial_s +\foG_s)\fob_s(\XintM_s) \rmd s +\sqrt{2} \int_{t_k}^t D_x \fob_s(\XintM_s) \rmd W_s\eqsp }^2\Bigg]\\
        &\lesssim   \PE\Bigg[\norm{\int_{t_k}^t (\partial_s +\foG_s)\fob_s(\XintM_s) \rmd s }^2\Bigg] + 2 \int_{t_k}^{t_{k+1}} \PE\Big[\norm{D_x \fob_s(\XintM_s)}^2\Big] \rmd s\eqsp.
        \end{split}
    \end{equation}We now bound separately the two upper addends. To do so, we introduce the auxiliary measures $\lambda_k^h(\rmd s)\in  \mathcal{P}([t_k, t_{k+1}])$ for $k=0,..., N-2$ and $\lambda_{N-1}^h(\rmd s)\in  \mathcal{P}([1-h, 1-\epsilon_1])$ which will help us, via a double change of measure argument, to mitigate the bad behaviour at $t=0$ and $t=1$ of the reciprocal characteristic of the mimicking drift (\ie, $\partial_s +\foG_s$), which is the trickiest addend. Namely, for $k=0,..., N-2$ we consider the measures $\lambda_k^h(\rmd s)\in  \mathcal{P}([t_k, t_{k+1}])$ defined as
    \begin{equation}\label{Auxiliary_measure_on_time}
        \lambda_k^h(\rmd s)=  \frac{\uprho(s)^{-1}}{Z_k}\Id_{[t_k,t_{k+1}]} \rmd s\eqsp,
    \end{equation}
    with \begin{equation}\label{def:uprho}
        \uprho(s)^{-1}= s^{-1/2}\Id_{\{s\le 1/2\}}+(1-s)^{-1/2}\Id_{\{s> 1/2\}}\eqsp,
    \end{equation} and
    \begin{equation}
        Z_k = \int_{\min\{t_k, 1/2\}}^{\min\{t_{k+1}, 1/2\}} r^{-1/2} \rmd r+ \int_{\max\{t_k,1/2\}}^{\max\{t_{k+1}, 1/2\}} (1-r)^{-1/2} \rmd r\eqsp.
    \end{equation} Whereas, for $k=N-1$, we consider the measure $\lambda_{N-1}^h(\rmd s)\in  \mathcal{P}([1-h, 1-\epsilon_1])$ defined as
    \begin{equation}
        \lambda_{N-1}^h(\rmd s)= \frac{\uprho(s)^{-1}}{Z_{N-1}}\Id_{[1-h,1-\epsilon_1]} \rmd s\eqsp,
    \end{equation}
    with
    $\uprho(s)^{-1}$ as in \eqref{def:uprho} and
    \begin{equation}
        Z_{N-1} = \int_{\min\{1-h, 1/2\}}^{\min\{1-\epsilon_1, 1/2\}} r^{-1/2} \rmd r+ \int_{\max\{1-h,1/2\}}^{\max\{1-\epsilon_1, 1/2\}} (1-r)^{-1/2} \rmd r\eqsp.
    \end{equation} Note that, for any $s\in [0, 1-\epsilon_1]$ and for any $k=0,..., N-1$, they hold
    \begin{equation}\label{properties_of_auxiliary_measure_on_time}
        \uprho(s)\lesssim 1\eqsp, \quad Z_k\lesssim h^{1/2} \eqsp, \quad \int_{0}^{1-\epsilon_1} \uprho^{-1}(s) \rmd s\lesssim 1\eqsp.
    \end{equation}
    We start by bounding the first addend, that is the one that involves the reciprocal characteristic of the mimicking drift. With a first change of measure argument, we get for any $k=0,..., N-1$,
    \begin{equation}
        \PE\Bigg[\norm{\int_{t_k}^t (\partial_s +\foG_s)\fob_s(\XintM_s) \rmd s }^2\Bigg] = Z_k^2 \PE\Bigg[\norm{\int_{t_k}^t (\partial_s +\foG_s)\fob_s(\XintM_s) \uprho(s)\lambda_k^h(\rmd s) }^2\Bigg]\eqsp,
    \end{equation}where, in the last inequality, we used \eqref{properties_of_auxiliary_measure_on_time}. But then, if we apply Jensen inequality and use an other change of measure argument, we get

        \begin{align}\label{bd:first_addend}
            \PE\Bigg[\norm{\int_{t_k}^t (\partial_s +\foG_s)\fob_s(\XintM_s) \rmd s }^2\Bigg]&\le Z_k^2 \PE\Bigg[\int_{t_k}^t \norm{(\partial_s +\foG_s)\fob_s(\XintM_s)}^2 \uprho(s)^{2} \lambda_k^h(\rmd s)\Bigg]\\
            &\le Z_k \int_{t_k}^{t}\PE\Big[\norm{(\partial_s +\foG_s)\fob_s(\XintM_s)}^2 \Big]\uprho(s) \rmd s\\
            &\lesssim h^{1/2}\int_{t_k}^{t_{k+1}} \PE\Big[\norm{(\partial_s +\foG_s)\fob_s(\XintM_s)}^2 \Big]\uprho(s) \rmd s\eqsp.
        \end{align}
    Let us now focus on the second addend. Remarkably, this addend can be bounded via the reciprocal characteristic of $\tbZ$: because of Ito's formula, for $t\in [0,1-\epsilon_1]$, it holds true
    \begin{equation}
        \rmd \norm{\fob_t(\XintM_t)}^2 = \Big\{2 \langle \fob_t, (\partial_t +\foG_t)\fob_t\rangle + 2\norm{D_x\fob_t}^2 \Big\}(\XintM_t)\rmd t + 2\sqrt{2} \langle \fob_t, D_x\fob_t\rangle(\XintM_t)\rmd W_t\eqsp.
    \end{equation}Note that, under \Cref{ass_moment} for $\mu$ and $\nustar$ and \Cref{ass_score} for $\tilde{\pi}$, Lemma 4 in \cite{gentiloni2024theoretical} ensures that the process $(\int_0^s \langle \fob_t, D_x\fob_t\rangle(\XintM_t)\rmd W_t)_{s\in [0,1-\epsilon_1]}$ is a true martingale. Consequently, we have that
    \begin{equation}\label{RHS_2}
    \begin{split}
        &2 \int_{t_k}^{t_{k+1}} \PE\Big[\norm{D_x\fob_s(\XintM_s)}^2\Big]\rmd s \\
        &\le \PE\Big[\norm{\fob_{t_{k+1}}(\XintM_{t_{k+1}})}^2\Big]- \PE\Big[\norm{\fob_{t_{k}}(\XintM_{t_{k}})}^2\Big] +2\Big| \int_{t_k}^{t_{k+1}} \PE[ \langle \fob_s (\XintM_s),(\partial_s +\foG_s)\fob_s (\XintM_s)\rangle] \rmd s \Big|\eqsp.
    \end{split}
    \end{equation}But then, using, as before, a double change of measure argument and applying Cauchy-Schwartz inequality, we can bound the above expression as follows
    \begin{equation}
        \begin{split}
            &2 \int_{t_k}^{t_{k+1}} \PE\Big[\norm{D_x\fob_s(\XintM_s)}^2\Big]\rmd s \\
        &\le \PE\Big[\norm{\fob_{t_{k+1}}(\XintM_{t_{k+1}})}^2\Big]- \PE\Big[\norm{\fob_{t_{k}}(\XintM_{t_{k}})}^2\Big] +2 Z_k\Big| \int_{t_k}^{t_{k+1}} \PE[ \langle \fob_s (\XintM_s),(\partial_s +\foG_s)\fob_s (\XintM_s)\rangle] \uprho(s) \lambda_k^h(\rmd s) \Big|\\
        &\le\PE\Big[\norm{\fob_{t_{k+1}}(\XintM_{t_{k+1}})}^2\Big]- \PE\Big[\norm{\fob_{t_{k}}(\XintM_{t_{k}})}^2\Big] + Z_k\int_{t_k}^{t_{k+1}} \PE\Big[\norm{\fob_s(\XintM_s)}^2\Big] \lambda_k^h(\rmd s)\\
        &\quad + Z_k\int_{t_k}^{t_{k+1}} \PE\Big[\norm{(\partial_s +\foG_s)\fob_s (\XintM_s)}^2 \Big] \uprho(s)^{2} \lambda_k^h(\rmd s)\\
        &= \PE\Big[\norm{\fob_{t_{k+1}}(\XintM_{t_{k+1}})}^2\Big]- \PE\Big[\norm{\fob_{t_{k}}(\XintM_{t_{k}})}^2\Big] + \int_{t_k}^{t_{k+1}} \PE\Big[\norm{\fob_s(\XintM_s)}^2\Big] \uprho(s)^{-1} \rmd s\\
        &\quad + \int_{t_k}^{t_{k+1}} \PE\Big[\norm{(\partial_s +\foG_s)\fob_s (\XintM_s)}^2 \Big] \uprho(s) \rmd s\eqsp.
        \end{split}
    \end{equation}Plugging this bound and \eqref{bd:first_addend} in \eqref{Decomposition_KL_Girsanov}, we get
        \begin{align}\label{Decomposition_KL_via_generator}
            &\KL(\nustar_{1-\epsilon_1}|\nu^{\theta^{\star}}_{1-\epsilon_1})\\
            &\lesssim \varepsilon^2  + h\PE\Big[\norm{\fob_{1-\epsilon_1}(\XintM_{1-\epsilon_1})}^2\Big]+h \int_{0}^{{1-\epsilon_1}}\PE\Big[ \norm{\fob_s (\XintM_s)}^2\Big]\uprho(s)^{-1}\rmd s\\
         &\quad +h (h^{1/2}+1)\int_{0}^{1-\epsilon_1}\PE\Big[ \norm{(\partial_s +\foG_s)\fob_s (\XintM_s)}^2\Big] \uprho(s)\rmd s\eqsp.
        \end{align}
    We now compute explicitly and upper bound each term appearing in the RHS of \eqref{Decomposition_KL_via_generator}, recalling that, because of Theorem 1 in \cite{gentiloni2024theoretical}, for any $s\in [0, 1]$, $\nustar_s=\text{Law}(X^\rmI_s)$.
    We start with the second term, that is
    \begin{equation}
    \begin{split}
       h\PE\Big[ \norm{\fob_{1-\epsilon_1} (\XintM_{1-\epsilon_1})}^2\Big] \eqsp.
    \end{split}
    \end{equation}
    Using \eqref{simmetry_nabla_heat_kernel}, integration by part, Jensen inequality, \Cref{lemma:pi_int_implies_marg_int} and $\moment[8]{\mu},\moment[8]{\nustar}\lesssim d^4$, we get
    \begin{equation}\label{term_2_dec_KL_gen}
    \begin{split}
        &\PE\Big[ \norm{\fob_{1-\epsilon_1} (\XintM_{1-\epsilon_1})}^2\Big]\\
        &\int_{\rset^d}\norm{\frac{\int_{\rset^{2d}} p_{1-\epsilon_1} (x|x_0)\nabla_x p_{\epsilon_1}(x_1|x)\tilde{\pi}(x_0,x_1) \rmd x_1 \rmd x_0}{p^{\rmI}_{1-\epsilon_1}(x)}}^2 p^{\rmI}_{1-\epsilon_1}(x)\rmd x\\
        &\lesssim\int_{\rset^d}\norm{\frac{\int_{\rset^{2d}} p_{1-\epsilon_1}(x|x_0) p_{\epsilon_1}(x_1|x) (\nabla_{x_1}\tilde{\pi}(x_0,x_1)/\tilde{\pi}(x_0,x_1))\tilde{\pi}(x_0,x_1)\rmd x_0 \rmd x_1}{p^{\rmI}_{1-\epsilon_1}(x)}}^2 p^{\rmI}_{1-\epsilon_1}(x)\rmd x \\
        &= \PE\Bigg[\norm{\PE\Bigg[\frac{\nabla_{x_1}\tilde{\pi}}{\tilde{\pi}}(X_0^{\rmI},X_1^{\rmI})\Bigg| X^{\rmI}_{1-\epsilon_1}\Bigg]}^2\Bigg]  \le  \PE\Bigg[\PE\Bigg[\norm{\frac{\nabla_{x_1}\tilde{\pi}}{\tilde{\pi}}(X_0^{\rmI},X_1^{\rmI})}^2\Bigg|X^{\rmI}_{1-\epsilon_1}\Bigg]\Bigg] \\
        &= \PE\Bigg[\norm{\frac{\nabla_{x_1}\tilde{\pi}}{\tilde{\pi}}(X_0^{\rmI},X_1^{\rmI})}^2\Bigg] =\norm{\frac{\nabla_{x_1}\tilde{\pi}}{\tilde{\pi}}}^2_{\mrl^2(\pi)}\le \norm{\nabla \log \tilde{\pi}}^2_{\mrl^2(\pi)}\lesssim \norm{\nabla \log \pi}^2_{\mrl^2(\pi)}+\moment[2]{\mu}+\moment[2]{\nustar}\\
        &\lesssim \norm{\nabla \log \pi}^2_{\mrl^2(\pi)}+d\eqsp.
    \end{split}
\end{equation}
Hence, we can bound the second term in \eqref{Decomposition_KL_via_generator} as follows
\begin{align}\label{eq:second_term}
    h\PE\Big[ \norm{\fob_{1-\epsilon_1} (\XintM_{1-\epsilon_1})}^2\Big] \lesssim h \left(\norm{\nabla \log \pi}^2_{\mrl^2(\pi)}+d\right)\eqsp.
\end{align}
We now deal with the third term of the RHS of \eqref{Decomposition_KL_via_generator}, \ie,
\begin{align}
    h \int_{0}^{{1-\epsilon_1}}\PE\Big[ \norm{\fob_s (\XintM_s)}^2\Big]\uprho(s)^{-1}\rmd s\eqsp.
\end{align}
Note that, proceeding as before and using \eqref{properties_of_auxiliary_measure_on_time}, we have that
  \begin{equation}\label{term_3_dec_KL_gen}
    \begin{split}
        &\int_{0}^{1-\epsilon_1}\PE\Big[ \norm{\fob_s (\XintM_s)}^2\Big]\uprho(s)^{-1}\rmd s\\
        &\lesssim \int_{0}^{1-\epsilon_1} \uprho(s)^{-1}\int_{\rset^d}\norm{\frac{\int_{\rset^{2d}} p_s(x|x_0) \nabla_x p_{1-s} (x_1|x)\tilde{\pi}(x_0,x_1) \rmd x_0 \rmd x_1}{p^{\rmI}_s(x)}}^2 p^{\rmI}_s(x)\rmd x \rmd s\\
        &= \int_{0}^{1-\epsilon_1} \uprho(s)^{-1}\int_{\rset^d}\norm{\frac{1}{p^\rmI_s(x)}\int_{\rset^{2d}} p_s(x|x_0) p_{1-s} (x_1| x)\frac{\nabla_{x_1}\tilde{\pi}}{\tilde{\pi}}(x_0, x_1)\tilde{\pi}(x_0, x_1) \rmd x_0, \rmd x_1}^2 p^{\rmI}_s(x)\rmd x\\
        &= \int_{0}^{1-\epsilon_1} \uprho(s)^{-1}\PE\Bigg[\norm{\PE\Bigg[\frac{\nabla_{x_1}\tilde{\pi}}{\tilde{\pi}}(X_0^{\rmI},X_1^{\rmI})\Bigg| X^{\rmI}_s \Bigg]}^2\Bigg] \rmd s\\
        &\le \int_{0}^{1-\epsilon_1} \uprho(s)^{-1}\PE\Bigg[\PE\Bigg[\norm{\frac{\nabla_{x_1}\tilde{\pi}}{\tilde{\pi}}(X_0^{\rmI},X_1^{\rmI})}^2\Bigg| X^{\rmI}_s \Bigg]\Bigg] \rmd s \\
        &=\int_{0}^{1-\epsilon_1} \uprho(s)^{-1}\PE\Bigg[\norm{\frac{\nabla_{x_1}\tilde{\pi}}{\tilde{\pi}}(X_0^{\rmI},X_1^{\rmI})}^2\Bigg] \rmd s\\
        &= \int_{0}^{1-\epsilon_1} \uprho(s)^{-1} \rmd s\:\norm{\nabla \log \tilde{\pi}}^2_{\mrl^2(\pi)}\lesssim \norm{\nabla \log \tilde{\pi}}^2_{\mrl^2(\pi)}\lesssim \norm{\nabla \log \pi}^2_{\mrl^2(\pi)}+\moment[2]{\mu}+\moment[2]{\nustar} \\
        &\lesssim \norm{\nabla \log \pi}^2_{\mrl^2(\pi)}+d\eqsp.
    \end{split}
\end{equation}Therefore, we can bound the third term in \eqref{Decomposition_KL_via_generator} as follows
\begin{align}\label{eq:third_term}
    h \int_{0}^{{1-\epsilon_1}}\PE\Big[ \norm{\fob_s (\XintM_s)}^2\Big]\uprho(s)^{-1}\rmd s\lesssim h\left(\norm{\nabla \log \pi}^2_{\mrl^2(\pi)}+d\right) \eqsp.
\end{align}
We now turn to the last term appearing in the r.h.s. of \eqref{Decomposition_KL_via_generator}, \ie,
\begin{align}
        &h (h^{1/2}+1) \int_{0}^{1-\epsilon_1}\PE\Big[ \norm{(\partial_s +\foG_s)\fob_s (\XintM_s)}^2\Big] \uprho(s)\rmd s\\
        &= h (h^{1/2}+1) \int_{0}^{1-\epsilon_1}\int_{\rset^{d}} \norm{(\partial_s +\foG_s)\fob_s (x)}^2 \uprho(s) p^{\rmI}_s(x) \rmd x \;\rmd s\eqsp.
    \end{align}
    Some computations and \eqref{fp_eq_heat_kernel} (we refer to Appendix A4 in \cite{gentiloni2024theoretical} for more details) lead to
    \begin{equation}\label{ito_drift_1}
        \begin{split}
             &(\partial_s +\foG_s)\fob_s(x)=\sum_{k=1}^{6} A_s^k(x)\eqsp,
        \end{split}
    \end{equation}where we have defined
    \begin{multline}
        A_s^1(x)=4\frac{\int_{\rset^{2d}} \Delta_x p_s(x|x_0) \nabla_x p_{1-s}(x_1|x) \tilde{\pi}(x_0,x_1) \rmd x_0 \rmd x_1}{p^{\rmI}_s(x)} \\
        - 4\frac{\int_{\rset^{2d}} \Delta_x p_s(x|x_0)  p_{1-s}(x_1|x) \tilde{\pi}(x_0,x_1) \rmd x_0 \rmd x_1 \int_{\rset^{2d}} p_s(x|x_0) \nabla_x p_{1-s}(x_1|x) \tilde{\pi}(x_0,x_1) \rmd x_0 \rmd x_1}{(p^{\rmI}_s(x))^2}\eqsp,
    \end{multline}
    \begin{multline}
        A_s^2(x)= 4\frac{\int_{\rset^{2d}} \nabla^2_x p_{1-s}(x_1|x) \nabla_x p_s(x|x_0)  \tilde{\pi}(x_0,x_1) \rmd x_0 \rmd x_1}{p^{\rmI}_s(x)}\\
        - 4\frac{\int_{\rset^{2d}} p_s(x|x_0) \nabla^2_x p_{1-s}(x_1|x) \tilde{\pi}(x_0,x_1) \rmd x_0 \rmd x_1 \int_{\rset^{2d}} \nabla_x p_s(x|x_0)  p_{1-s}(x_1|x) \tilde{\pi}(x_0,x_1) \rmd x_0 \rmd x_1}{(p^{\rmI}_s(x))^2}\eqsp.
    \end{multline}
    \begin{multline}
        A^3_s(x)=-4\frac{\int_{\rset^{2d}} \nabla_x p_{1-s}(x_1|x)(\nabla_x p_s(x|x_0))^{\transpose} \tilde{\pi}(x_0,x_1) \rmd x_0 \rmd x_1}{p^{\rmI}_s(x)}\\
        \cdot\frac{ \int_{\rset^{2d}} \nabla_x p_s(x|x_0)  p_{1-s}(x_1|x) \tilde{\pi}(x_0,x_1) \rmd x_0 \rmd x_1}{p^{\rmI}_s(x)}\eqsp,
        \end{multline}
    \begin{multline}
    A^4_s(x)= -4\frac{\int_{\rset^{2d}} \langle \nabla_x p_{1-s}(x_1|x), \nabla_x p_s(x|x_0)\rangle \tilde{\pi}(x_0,x_1) \rmd x_0 \rmd x_1}{p^{\rmI}_s(x)}\\
        \cdot\frac{\int_{\rset^{2d}}  p_s(x|x_0) \nabla_x p_{1-s}(x_1|x) \tilde{\pi}(x_0,x_1) \rmd x_0 \rmd x_1}{p^{\rmI}_s(x)}\eqsp,
    \end{multline}
    \begin{multline}
        A_s^5(x)= 4\frac{\norm{\int_{\rset^{2d}} \nabla_x p_s(x|x_0) p_{1-s}(x_1|x) \tilde{\pi}(x_0,x_1) \rmd x_0 \rmd x_1}^2 }{(p^{\rmI}_s(x))^2}\\
        \cdot\frac{\int_{\rset^{2d}} p_s(x|x_0) \nabla_x p_{1-s}(x_1|x) \tilde{\pi}(x_0,x_1) \rmd x_0 \rmd x_1}{p^{\rmI}_s(x)}\eqsp,
        \end{multline}
        \begin{multline}
    A_s^6(x)=4\frac{\int_{\rset^{2d}} p_s(x|x_0) \nabla_x p_{1-s}(x_1|x) \tilde{\pi}(x_0,x_1) \rmd x_0 \rmd x_1}{p^{\rmI}_s(x)} \\
    \cdot\Bigg(\frac{\int_{\rset^{2d}} \nabla_x p_s(x|x_0) p_{1-s}(x_1|x) \tilde{\pi}(x_0,x_1) \rmd x_0 \rmd x_1}{p^{\rmI}_s(x)}\Bigg)^{\transpose}\frac{\int_{\rset^{2d}} p_s(x|x_0) \nabla_x p_{1-s}(x_1|x) \tilde{\pi}(x_0,x_1) \rmd x_0 \rmd x_1}{p^{\rmI}_s(x)} \eqsp.
    \end{multline}
     Therefore
    \begin{align}
        h (h^{1/2}+1)\int_{0}^{1-\epsilon_1}\PE\Big[ \norm{(\partial_s +\foG_s)\fob_s (\XintM_s)}^2\Big]\uprho(s)\rmd s\lesssim h (h^{1/2}+1)\sum_{k=1}^6 \int_{0}^{1-\epsilon_1} \int_{\rset^d}\norm{A_s^k(x)}^2 \uprho(s) p_s^{\rmI} (x)\rmd x \rmd s\eqsp.
    \end{align}
Plugging the above inequality, \eqref{eq:second_term} and \eqref{eq:third_term} in \eqref{Decomposition_KL_via_generator}, we obtain
\begin{align}\label{eq:mid_decomposition_kl}
\KL(\nu_{1-\epsilon_1}|\applaw_{1-\epsilon_1})\lesssim \varepsilon^2  +h\left(\norm{\nabla \log \pi}^2_{\mrl^2(\pi)}+d\right)+h (h^{1/2}+1)\sum_{k=1}^6 \int_{0}^{1-\epsilon_1} \int_{\rset^d}\norm{A_s^k(x)}^2 \uprho(s) p^{\rmI}_s (x)\rmd x \rmd s\eqsp.
        \end{align}

We now bound each term $A_s^k$ in the sum, starting with $A_s^1$. Integrating by parts and using \eqref{simmetry_nabla_heat_kernel} and \eqref{simmetry_delta_heat_kernel}, we get
\begin{align}
    &\int_{0}^{1-\epsilon_1} \int_{\rset^d}\norm{A_s^1(x)}^2  \uprho(s)p_s^{\rmI} (x)\rmd x \rmd s\; \\
    &\lesssim\int_0^{1-\epsilon_1} \int_{\rset^d} \left\|\frac{\int_{\rset^{2d}} \Delta_x p_s(x|x_0) \nabla_x p_{1-s}(x_1|x) \tilde{\pi}(x_0,x_1) \rmd x_0 \rmd x_1}{p^{\rmI}_s(x)}\right.\\
    &\left.\quad - \frac{\int_{\rset^{2d}} \Delta_x p_s(x|x_0)  p_{1-s}(x_1|x) \tilde{\pi}(x_0,x_1) \rmd x_0 \rmd x_1}{p^{\rmI}_s(x)}\right.\\
    &\left. \quad \cdot\frac{\int_{\rset^{2d}} p_s(x|x_0) \nabla_x p_{1-s}(x_1|x) \tilde{\pi}(x_0,x_1) \rmd x_0 \rmd x_1}{p^{\rmI}_s(x)}\right\|^2 \uprho(s)p_s^{\rmI} (x)\rmd x \rmd s\\
    &=\int_{0}^{1-\epsilon_1} \int_{\rset^d} \left\|\frac{\int_{\rset^{2d}}  \ps{\nabla_x p_s(x|x_0)}{\nabla_{x_0}\tilde{\pi}(x_0,x_1)} \nabla_x p_{1-s}(x_1|x) \rmd x_0 \rmd x_1}{p^{\rmI}_s(x)}\right.\\
    &\left.\quad - \frac{\int_{\rset^{2d}} \ps{\nabla_x p_s(x|x_0)}{\nabla_{x_0}\tilde{\pi}(x_0,x_1)} p_{1-s}(x_1|x)  \rmd x_0 \rmd x_1}{p^{\rmI}_s(x)}\right.\\
    &\left. \quad \cdot\frac{\int_{\rset^{2d}} p_s(x|x_0) \nabla_x p_{1-s}(x_1|x) \tilde{\pi}(x_0,x_1) \rmd x_0 \rmd x_1}{p^{\rmI}_s(x)}\right\|^2 \uprho(s) p_s^{\rmI} (x)\rmd x \rmd s\\
    &=\int_{0}^{1-\epsilon_1} \int_{\rset^d} \left\|\frac{\int_{\rset^{2d}}  \ps{\frac{x-x_0}{s}}{\frac{\nabla_{x_0}\tilde{\pi}}{\tilde{\pi}}}(x_0,x_1)\frac{x_1-x}{1-s} p_s(x|x_0) p_{1-s}(x_1|x)\tilde{\pi}(x_0,x_1) \rmd x_0 \rmd x_1}{p^{\rmI}_s(x)}\right.\\
    &\left.\quad - \frac{\int_{\rset^{2d}} \ps{\frac{x-x_0}{s}}{\frac{\nabla_{x_0}\tilde{\pi}}{\tilde{\pi}}(x_0,x_1)} p_s(x|x_0) p_{1-s}(x_1|x) \tilde{\pi}(x_0,x_1) \rmd x_0 \rmd x_1}{p^{\rmI}_s(x)}\right.\\
    &\left. \quad \cdot\frac{\int_{\rset^{2d}} \frac{x_1-x}{1-s} p_s(x|x_0)  p_{1-s}(x_1|x) \tilde{\pi}(x_0,x_1) \rmd x_0 \rmd x_1}{p^{\rmI}_s(x)}\right\|^2 \uprho(s) p_s^{\rmI} (x)\rmd x \rmd s\\
    &=\int_{0}^{1-\epsilon_1} \PE\left[\left\|\PE\left[\ps{\frac{X^\mathrm{I}_s-X^\mathrm{I}_0}{s}}{\frac{\nabla_{x_0}\tilde{\pi}}{\tilde{\pi}}(X^\mathrm{I}_0,X^\mathrm{I}_1)}\frac{X^\mathrm{I}_1-X^\mathrm{I}_s}{1-s}\Big|X^\mathrm{I}_s\right]\right.\right.\\
    &\left.\left. \quad -\PE\left[\ps{\frac{X^\mathrm{I}_s-X^\mathrm{I}_0}{s}}{\frac{\nabla_{x_0}\tilde{\pi}}{\tilde{\pi}}(X^\mathrm{I}_0,X^\mathrm{I}_1)}\Big|X^\mathrm{I}_s\right]\PE\left[\frac{X^\mathrm{I}_1-X^\mathrm{I}_s}{1-s}\Big|X^\mathrm{I}_s\right]\right\|^2\right]\uprho(s)\rmd s\\
    &\lesssim\label{eq:where_to_plug_measurability_property_in_conditional_expectation} \int_{0}^{1-\epsilon_1} \left\{\PE\left[\left\|\PE\left[\ps{\frac{X^\mathrm{I}_s-X^\mathrm{I}_0}{s}}{\frac{\nabla_{x_0}\tilde{\pi}}{\tilde{\pi}}(X^\mathrm{I}_0,X^\mathrm{I}_1)}\frac{X^\mathrm{I}_1-X^\mathrm{I}_s}{1-s}\Big|X^\mathrm{I}_s\right]\right\|^2\right]\right.\\
    &\quad \left.+\PE\left[\left\|\PE\left[\ps{\frac{X^\mathrm{I}_s-X^\mathrm{I}_0}{s}}{\frac{\nabla_{x_0}\tilde{\pi}}{\tilde{\pi}}(X^\mathrm{I}_0,X^\mathrm{I}_1)}\Big|X^\mathrm{I}_s\right]\PE\left[\frac{X^\mathrm{I}_1-X^\mathrm{I}_s}{1-s}\Big|X^\mathrm{I}_s\right]\right\|^2\right]\right\}\uprho(s)\rmd s\eqsp.
\end{align}

At this point, we split the time interval $[0,1-\epsilon_1]$ in two, $[0,1/2]$ and $[1/2, 1-\epsilon_1]$ and we rewrite the RHS as
\begin{align}
    &\int_{0}^{1-\epsilon_1} \int_{\rset^d}\norm{A_s^1(x)}^2  \uprho(s) p_s^{\rmI} (x)\rmd x \rmd s\; \\
    &\lesssim\int_{0}^{1/2} \left\{\PE\left[\left\|\ps{\frac{X^\mathrm{I}_s-X^\mathrm{I}_0}{s}}{\frac{\nabla_{x_0}\tilde{\pi}}{\tilde{\pi}}(X^\mathrm{I}_0,X^\mathrm{I}_1)}X^\mathrm{I}_1-X^\mathrm{I}_s\right\|^2\right]\right.\\
    &\quad \left.+\PE\left[\left\|\PE\left[\ps{\frac{X^\mathrm{I}_s-X^\mathrm{I}_0}{s}}{\frac{\nabla_{x_0}\tilde{\pi}}{\tilde{\pi}}(X^\mathrm{I}_0,X^\mathrm{I}_1)}\Big|X^\mathrm{I}_s\right]\PE\left[X^\mathrm{I}_1-X^\mathrm{I}_s\Big|X^\mathrm{I}_s\right]\right\|^2\right]\right\}\uprho(s)\rmd s\\
    &\quad + \int_{1/2}^{1-\epsilon_1} \left\{\PE\left[\left\|\ps{X^\mathrm{I}_s-X^\mathrm{I}_0}{\frac{\nabla_{x_0}\tilde{\pi}}{\tilde{\pi}}(X^\mathrm{I}_0,X^\mathrm{I}_1)}\frac{X^\mathrm{I}_1-X^\mathrm{I}_s}{1-s}\right\|^2\right]\right.\\
    &\quad \left. +\PE\left[\left\|\PE\left[\ps{X^\mathrm{I}_s-X^\mathrm{I}_0}{\frac{\nabla_{x_0}\tilde{\pi}}{\tilde{\pi}}(X^\mathrm{I}_0,X^\mathrm{I}_1)}\Big|X^\mathrm{I}_s\right]\PE\left[\frac{X^\mathrm{I}_1-X^\mathrm{I}_s}{1-s}\Big|X^\mathrm{I}_s\right]\right\|^2\right]\right\}\uprho(s)\rmd s\eqsp.
\end{align}We first focus on the time sub-interval $s\in[0,1/2]$. To this aim, we fix $s\in [0,1/2]$. Holder and Jensen inequalities yield

\begin{align}
    & \PE\left[\left\|\ps{\frac{X^\mathrm{I}_s-X^\mathrm{I}_0}{s}}{\frac{\nabla_{x_0}\tilde{\pi}}{\tilde{\pi}}(X^\mathrm{I}_0,X^\mathrm{I}_1)}X^\mathrm{I}_1-X^\mathrm{I}_s\right\|^2\right]\\
    &\quad +\PE\left[\left\|\PE\left[\ps{\frac{X^\mathrm{I}_s-X^\mathrm{I}_0}{s}}{\frac{\nabla_{x_0}\tilde{\pi}}{\tilde{\pi}}(X^\mathrm{I}_0,X^\mathrm{I}_1)}\Big|X^\mathrm{I}_s\right]\PE\left[X^\mathrm{I}_1-X^\mathrm{I}_s\Big|X^\mathrm{I}_s\right]\right\|^2\right]\\
    &\le \PE\left[\left\|\ps{\frac{X^\mathrm{I}_s-X^\mathrm{I}_0}{s}}{\frac{\nabla_{x_0}\tilde{\pi}}{\tilde{\pi}}(X^\mathrm{I}_0,X^\mathrm{I}_1)}\right\|^4\right]^{\frac{1}{2}}\PE\left[\left\|X^\mathrm{I}_1-X^\mathrm{I}_s\right\|^{4}\right]^{\frac{1}{2}}\\
    &\quad +\PE\left[\left\|\PE\left[\ps{\frac{X^\mathrm{I}_s-X^\mathrm{I}_0}{s}}{\frac{\nabla_{x_0}\tilde{\pi}}{\tilde{\pi}}(X^\mathrm{I}_0,X^\mathrm{I}_1)}\Big|X^\mathrm{I}_s\right]\right\|^4\right]^{\frac{1}{2}}\PE\left[\left\|\PE\left[X^\mathrm{I}_1-X^\mathrm{I}_s\Big|X^\mathrm{I}_s\right]\right\|^4\right]^{\frac{1}{2}}\\
    &\lesssim \PE\left[\left\|\ps{\frac{X^\mathrm{I}_s-X^\mathrm{I}_0}{s}}{\frac{\nabla_{x_0}\tilde{\pi}}{\tilde{\pi}}(X^\mathrm{I}_0,X^\mathrm{I}_1)}\right\|^4\right]^{\frac{1}{2}}\PE\left[\left\|X^\mathrm{I}_1-X^\mathrm{I}_s\right\|^{4}\right]^{\frac{1}{2}}\\
    &\le \PE\left[\left\|\frac{X^\mathrm{I}_s-X^\mathrm{I}_0}{s}\right\|^{8}\right]^{\frac{1}{4}}\PE\left[\left\|\frac{\nabla_{x_0}\tilde{\pi}}{\tilde{\pi}}(X^\mathrm{I}_0,X^\mathrm{I}_1)\right\|^{8}\right]^{\frac{1}{4}}\PE\left[\left\|X^\mathrm{I}_1-X^\mathrm{I}_s\right\|^{4}\right]^{\frac{1}{2}}\eqsp.
\end{align}
Therefore, we have that
\begin{align}
\int_{0}^{1/2} \int_{\rset^d}\norm{A_s^1(x)}^2  \uprho(s) p_s^{\rmI} (x)\rmd x \rmd s\; \lesssim \int_{0}^{1/2}\PE\left[\left\|\frac{X^\mathrm{I}_s-X^\mathrm{I}_0}{s}\right\|^{8}\right]^{\frac{1}{4}}\PE\left[\left\|\frac{\nabla_{x_0}\tilde{\pi}}{\tilde{\pi}}(X^\mathrm{I}_0,X^\mathrm{I}_1)\right\|^{8}\right]^{\frac{1}{4}}\PE\left[\left\|X^\mathrm{I}_1-X^\mathrm{I}_s\right\|^{4}\right]^{\frac{1}{2}}\uprho(s) \rmd s\eqsp.
\end{align}
Using \Cref{lemma_on_moments_bounded}, \Cref{lemma:lemma2_back_in_dfm}, \Cref{lemma:lemma3_back_in_dfm}, Equation \eqref{properties_of_auxiliary_measure_on_time}, \Cref{lemma:pi_int_implies_marg_int} and the fact that $\moment[8]{\mu}, \moment[8]{\nustar}\lesssim d^4$, we get that
\begin{align}
&\int_{0}^{1/2} \int_{\rset^d}\norm{A_s^1(x)}^2  \uprho(s) p_s^{\rmI} (x)\rmd x \rmd s\; \\
&\lesssim \int_{0}^{1/2}\PE\left[\left\|\frac{\baX_{1}-\baX_{1-s}}{s}\right\|^{8}\right]^{\frac{1}{4}}\PE\left[\left\|\frac{\nabla_{x_0}\tilde{\pi}}{\tilde{\pi}}(X^\mathrm{I}_0,X^\mathrm{I}_1)\right\|^{8}\right]^{\frac{1}{4}}\PE\left[\left\|X^\mathrm{I}_1-X^\mathrm{I}_s\right\|^{4}\right]^{\frac{1}{2}}\uprho(s) \rmd s\\
    &\lesssim \|\nabla\log \tilde{\pi}\|^2_{\mrl^{8}(\pi)} (d^2+\moment[4]{\mu}+\moment[4]{\nustar})^{\frac{1}{2}}\int_{0}^{1/2}\PE\left[\left\|\frac{\overleftarrow{f}_{1-s}^s+\overleftarrow{g}_{1-s}^s}{s}\right\|^{8} \right]^{\frac{1}{4}}\uprho(s) \rmd s\\
    &\lesssim  \|\nabla\log \tilde{\pi}\|^2_{\mrl^{8}(\pi)} \left(d+\sqrt{\moment[4]{\mu}}+\sqrt{\moment[4]{\nustar}}\right)\int_{0}^{1/2}\left\{\PE\left[\left\|\frac{\overleftarrow{f}_{1-s}^s}{s}\right\|^{8}\right]^{\frac{1}{4}}+ \PE\left[\left\|\frac{\overleftarrow{g}_{1-s}^s}{s}\right\|^{8} \right]^{\frac{1}{4}}\right\}\uprho(s) \rmd s\\
    &\lesssim \|\nabla\log \tilde{\pi}\|^2_{\mrl^{8}(\pi)} \left(d+\sqrt{\moment[4]{\mu}}+\sqrt{\moment[4]{\nustar}}\right) \left(\norm{\nabla\log \pi}_{\mrl^8(\pi)}^2+\sqrt[4]{\moment[8]{\mu}}+\sqrt[4]{\moment[8]{\nustar}}+\int_{0}^{1/2}\{d^4 s^{4-8}\}^{\frac{1}{4}}s^{\frac{1}{2}} \rmd s\right)\\
    &=\|\nabla\log \tilde{\pi}\|^2_{\mrl^{8}(\pi)} \left(d+\sqrt{\moment[4]{\mu}}+\sqrt{\moment[4]{\nustar}}\right) \left(\norm{\nabla\log \pi}_{\mrl^8(\pi)}^2+\sqrt[4]{\moment[8]{\mu}}+\sqrt[4]{\moment[8]{\nustar}}+d \int_{0}^{1/2} s^{-\frac{1}{2}}\rmd s \right)\\
    &\lesssim \|\nabla\log \tilde{\pi}\|^2_{\mrl^{8}(\pi)} \left(d+\sqrt{\moment[4]{\mu}}+\sqrt{\moment[4]{\nustar}}\right) \left(\norm{\nabla\log \pi}_{\mrl^8(\pi)}^2+\sqrt[4]{\moment[8]{\mu}}+\sqrt[4]{\moment[8]{\nustar}}+d\right)\\
    &\lesssim \left(\norm{\nabla\log \pi}_{\mrl^8(\pi)}^2+\sqrt[4]{\moment[8]{\mu}}+\sqrt[4]{\moment[8]{\nustar}}\right) \left(d+\sqrt{\moment[4]{\mu}}+\sqrt{\moment[4]{\nustar}}\right) \left(\norm{\nabla\log \pi}_{\mrl^8(\pi)}^2+\sqrt[4]{\moment[8]{\mu}}+\sqrt[4]{\moment[8]{\nustar}}+d\right)\\
    &\lesssim\left(\norm{\nabla\log \pi}_{\mrl^8(\pi)}^2+d\right) d\left(\norm{\nabla\log \pi}_{\mrl^8(\pi)}^2+d\right)\lesssim d \left(d^2+\norm{\nabla\log \pi}_{\mrl^8(\pi)}^4\right)\eqsp.
    \end{align}
To handle with the time sub-interval $s\in[1/2, 1-\epsilon_1]$, we proceed in a similar way, but this time using \Cref{lemma:lemma2_for_in_dfm,lemma:lemma3_for_in_dfm} rather than \Cref{lemma:lemma2_back_in_dfm,lemma:lemma3_back_in_dfm}. By doing so we obtain that
\begin{align}
    \int_{1/2}^{1-\epsilon_1} \norm{A_s^1(x)}^2  \uprho(s) p_s^{\rmI} (x)\rmd x \rmd s\lesssim d \left(d^2+\norm{\nabla\log \pi}_{\mrl^8(\pi)}^4\right) \eqsp.
\end{align}
Putting together the two bounds we get
\begin{align}
     \int_{0}^{1-\epsilon_1} \int_{\rset^d}\norm{A_s^1(x)}^2  \uprho(s) p_s^{\rmI} (x)\rmd x \rmd s\lesssim d \left(d^2+\norm{\nabla\log \pi}_{\mrl^8(\pi)}^4\right) \eqsp.
\end{align}
Proceeding in a similar way (we omit the argument, as it is almost a duplication of the previous one), we get
\begin{align}
    \int_{0}^{1-\epsilon_1} \int_{\rset^d}\norm{A_s^2(x)}^2  \uprho(s) p_s^{\rmI} (x)\rmd x\; \rmd s\lesssim d \left(d^2+\norm{\nabla\log \pi}_{\mrl^8(\pi)}^4\right) \eqsp.
\end{align}
We now switch to $A_s^3$. Using \eqref{simmetry_delta_heat_kernel}, integration by parts, we have that
\begin{align}
     &\int_{0}^{1-\epsilon_1} \int_{\rset^d}\norm{A_s^3(x)}^2  \uprho(s)p_s^{\rmI} (x)\rmd x \rmd s\; \\
     &\lesssim \int_{0}^{1-\epsilon_1} \int_{\rset^d}\left\|\frac{\int_{\rset^{2d}} \nabla_x p_{1-s}(x_1|x)(\nabla_x p_s(x|x_0))^{\transpose} \tilde{\pi}(x_0,x_1) \rmd x_0 \rmd x_1}{p^{\rmI}_s(x)}\right.\\
     &\left. \quad \cdot \frac{ \int_{\rset^{2d}} \nabla_x p_s(x|x_0)  p_{1-s}(x_1|x) \tilde{\pi}(x_0,x_1) \rmd x_0 \rmd x_1}{p^{\rmI}_s(x)}\right\|^2\uprho(s)p_s^{\rmI} (x)\rmd x \rmd s\;\\
     & \lesssim\int_{0}^{1-\epsilon_1} \int_{\rset^d}\left\|\frac{\int_{\rset^{2d}} \frac{x_1-x}{1-s}\left(\frac{\nabla_{x_0}\tilde{\pi}}{\tilde{\pi}}(x_0,x_1) \right)^{\transpose} p_{1-s}(x_1|x) p_s(x|x_0) \tilde{\pi}(x_0,x_1) \rmd x_0 \rmd x_1}{p^{\rmI}_s(x)}\right.\\
     &\left. \quad \cdot\frac{ \int_{\rset^{2d}} \frac{x-x_0}{s} p_s(x|x_0)  p_{1-s}(x_1|x) \tilde{\pi}(x_0,x_1) \rmd x_0 \rmd x_1}{p^{\rmI}_s(x)}\right\|^2\uprho(s)p_s^{\rmI} (x)\rmd x \rmd s\;\\
     &=\int_{0}^{1-\epsilon_1} \PE\left[\left\|\PE\left[\frac{X^\rmI_1-X^\rmI_s}{1-s}\left(\frac{\nabla_{x_0}\tilde{\pi}}{\tilde{\pi}}\left(X^\rmI_0, X^\rmI_1\right)\right)^{\transpose}\Big| X^\rmI_s\right]\PE\left[\frac{X^\rmI_s-X^\rmI_0}{s}\Big| X^\rmI_s\right]\right\|^2\right]\uprho(s) \rmd s
\end{align}Mimicking the argument used to bound $A^1_s$ (we omit the details, as they are almost a duplication of the previous one), we eventually get that

\begin{align}
    \int_{0}^{1-\epsilon_1} \int_{\rset^d}\norm{A_s^3(x)}^2  \uprho(s)p_s^{\rmI} (x)\rmd x \rmd s\; \lesssim d \left(d^2+\norm{\nabla\log \pi}_{\mrl^8(\pi)}^4\right) \eqsp.
\end{align}
In the very same way, we get
\begin{align}
     &\int_{0}^{1-\epsilon_1} \int_{\rset^d}\norm{A_s^4(x)}^2  \uprho(s)p_s^{\rmI} (x)\rmd x \rmd s\; \lesssim d \left(d^2+\norm{\nabla\log \pi}_{\mrl^8(\pi)}^4\right) \eqsp.
\end{align}
We now turn to $A^5_s.$ Using \eqref{simmetry_delta_heat_kernel} and integration by parts, we obtain
\begin{align}
 &\int_{0}^{1-\epsilon_1} \int_{\rset^d}\norm{A_s^5(x)}^2  \uprho(s)p_s^{\rmI} (x)\rmd x \rmd s\; \\
&\lesssim\int_{0}^{1-\epsilon_1} \int_{\rset^d}\left\| \frac{\norm{\int_{\rset^{2d}} \nabla_x p_s(x|x_0) p_{1-s}(x_1|x) \tilde{\pi}(x_0,x_1) \rmd x_0 \rmd x_1}^2 }{(p^{\rmI}_s(x))^2}\right.\\
    &\left. \quad \cdot\frac{\int_{\rset^{2d}} p_s(x|x_0) \nabla_x p_{1-s}(x_1|x) \tilde{\pi}(x_0,x_1) \rmd x_0 \rmd x_1}{p^{\rmI}_s(x)} \right\|^2  \uprho(s)p_s^{\rmI} (x)\rmd x \rmd s\\
    &=\int_{0}^{1-\epsilon_1} \int_{\rset^d}\left\| \left\langle \frac{\int_{\rset^{2d}} \frac{x-x_0}{s} p_s(x|x_0) p_{1-s}(x_1|x) \tilde{\pi}(x_0,x_1) \rmd x_0 \rmd x_1 }{p^{\rmI}_s(x)},\right.\right.\\
    &\left.\left. \cdot\frac{\int_{\rset^{2d}} \frac{\nabla_{x_0}\tilde{\pi}}{\tilde{\pi}}(x_0,x_1) p_s(x|x_0) p_{1-s}(x_1|x) \tilde{\pi}(x_0,x_1) \rmd x_0 \rmd x_1 }{p^{\rmI}_s(x)}\right\rangle \right.\\
    &\left. \quad \cdot\frac{\int_{\rset^{2d}} \frac{x_1-x}{1-s}p_s(x|x_0) p_{1-s}(x_1|x) \tilde{\pi}(x_0,x_1) \rmd x_0 \rmd x_1}{p^{\rmI}_s(x)} \right\|^2  \uprho(s)p_s^{\rmI} (x)\rmd x \rmd s\\
    &=\int_{0}^{1-\epsilon_1} \PE\left[\left\|\ps{\PE\left[\frac{X^\rmI_s-X^\rmI_0}{s}\Big|X^\rmI_s\right]}{\PE\left[\frac{\nabla_{x_0}\tilde{\pi}}{\tilde{\pi}}\left(X^\rmI_0, X^\rmI_1\right)\Big|X^\rmI_s\right]}\PE\left[\frac{X^\rmI_1-X^\rmI_s}{1-s}\Big|X^\rmI_s\right]\right\|^2\right] \uprho(s) \rmd s\eqsp.
\end{align}
Proceeding as for $A^1_s$, we eventually get
\begin{align}
     \int_{0}^{1-\epsilon_1} \int_{\rset^d}\norm{A_s^5(x)}^2  \uprho(s)p_s^{\rmI} (x)\rmd x \rmd s\; \lesssim d \left(d^2+\norm{\nabla\log \pi}_{\mrl^8(\pi)}^4\right) \eqsp.
\end{align}
The argument to bound $A^6_s$ is almost the same (hence omitted) and leads to
\begin{align}
     \int_{0}^{1-\epsilon_1} \int_{\rset^d}\norm{A_s^6(x)}^2  \uprho(s)p_s^{\rmI} (x)\rmd x \rmd s\; \lesssim d \left(d^2+\norm{\nabla\log \pi}_{\mrl^8(\pi)}^4\right) \eqsp.
\end{align}
Putting together the bounds on the $\{A_s^k\}_{k=1}^{6}$ derived so far, we eventually obtain

\begin{align}\label{term_4_dec_KL_gen}
 (h^{1/2}+1)\sum_{k=1}^6 \int_{0}^{1-\epsilon_1} \int_{\rset^d}\norm{A_s^k(x)}^2 \uprho(s) p_s^Y (x)\rmd x \rmd s\lesssim h(h^{1/2}+1)d \left(d^2+\norm{\nabla\log \pi}_{\mrl^8(\pi)}^4\right)\eqsp.
\end{align}
Plugging the above inequality into \eqref{eq:mid_decomposition_kl}, we get \begin{align}
        \KL(\nustar_{1-\epsilon_1}|\nu^{\theta^{\star}}_{1-\epsilon_1})\lesssim \varepsilon^2 +h(h^{1/2}+1)d \left(d^2+\norm{\nabla\log \pi}_{\mrl^8(\pi)}^4\right)\eqsp.
    \end{align}
The byproduct of the above estimate and \eqref{lsc_of_kl} leads to
\begin{align}
    \KL(\nustar | \nu^{\theta^\star}_1)\lesssim \varepsilon^2
+ h(h^{1/2}+1) d \left(d^2+\norm{\nabla\log \pi}_{\mrl^8(\pi)}^4\right)\eqsp.
\end{align}
\end{proof}

\subsection{Early-stopping regime with constant step-size}
\begin{proof}[Proof of \Cref{theo_early}:]
Fix $0<\delta<1/2$. We want to apply \Cref{theo_no_early} to $\mu$, $\nustar_{1-\delta}$ and the coupling $\pi_{1-\delta}\in \Pi(\mu, \nustar_{1-\delta})$ defined as
\begin{align}
    \pi_{1-\delta}(x_0, x_{1-\delta})=\int_{\rset^d} p^{\rmI}_{1-\delta|0,1}(x_{1-\delta}|x_0, x_1)  \pi_{0|1}(x_0| x_1)\nustar(\rmd x_1)\eqsp,
\end{align}where $(x_0,x_1,x_{1-\delta}) \mapsto p^{\rmI}_{1-\delta|0,1}(x_{1-\delta}|x_0, x_1)$ denotes the density of $X^\rmI_{1-\delta}$ given $(X^\rmI_0, X^\rmI_1)$ with respect to the Lebesgue measure. To this aim, we need to prove that $\nustar_{1-\delta}$ has finite 8th order moment and that $\pi_{1-\delta}$ satisfies 
\begin{align}
    \|\nabla\log\pi_{1-\delta}\|_{\mathrm{L}^8(\pi_{1-\delta})}<+\infty\eqsp.
\end{align}
We start with $\nustar_{1-\delta}$. Note that, as a very consequence of \eqref{interpolant_in_linear_form}, we have that
\begin{align}\label{bound_moment_hat}
    \moment[8]{\nustar_{1-\delta}}= \PE\left[\norm{X^{\rmI}_{1-\delta}}^8\right]\lesssim \delta^8 \moment[8]{\mu}+(1-\delta)^8\moment[8]{\nustar}+ d^4 \delta^4(1-\delta)^4<+\infty\eqsp.
\end{align}

We now switch to $\pi_{1-\delta}$. 
It follows from the very definition of the stochastic interpolant that,
\begin{equation}
    p^{\rmI}_{1-\delta|0,1}(x_{1-\delta}|x_0, x_1)= \frac{1}{(4\pi \delta(1-\delta))^{d/2}} \exp\Bigg(-\frac{\norm{x_{1-\delta}-\delta x_0 -(1-\delta)x_1}^2}{4\delta(1-\delta)}\Bigg)\eqsp.
\end{equation}Therefore,
\begin{equation}
    \nabla_{x_0} p^{\rmI}_{1-\delta|0,1}(x_{1-\delta}|x_0, x_1) = \frac{x_{1-\delta}-\delta x_0 -(1-\delta)x_1}{2(1-\delta)}p^{\rmI}_{1-\delta|0,1}(x_{1-\delta}|x_0, x_1)\eqsp,
\end{equation}and
\begin{equation}
    \nabla_{x_{1-\delta}} p^{\rmI}_{1-\delta|0,1}(x_{1-\delta}|x_0, x_1)= -\frac{x_{1-\delta}-\delta x_0 -(1-\delta)x_1}{2\delta(1-\delta)}p^{\rmI}_{1-\delta|0,1}(x_{1-\delta}|x_0, x_1)\eqsp.
\end{equation} Furthermore,
\begin{equation}
\begin{split}
    \frac{p^{\rmI}_{1-\delta|0,1}(x_{1-\delta}|x_0, x_1)\pi_{0|1}(x_0| x_1)\nustar(\rmd x_1)}{\int_{\rset^d}p^{\rmI}_{1-\delta|0,1}(x_{1-\delta}|x_0, \tilde{x}_1)\pi_{0|1}(x_0| \tilde{x}_1)\nustar(\rmd \tilde{x}_1)}&= p^{\rmI}_{1|0, 1-\delta}(\rmd x_1|x_0, x_{1-\delta})\eqsp.
    \end{split}
\end{equation}Consequently, we have that
\begin{equation}
    \begin{split}
       &\frac{\nabla_{x_0}\pi_{1-\delta}}{\pi_{1-\delta}}(x_0, x_{1-\delta})\\
       &= \frac{\int_{\rset^d} p^{\rmI}_{1-\delta|0,1}(x_{1-\delta}|x_0, x_1)\nabla_{x_0}\pi_{0|1}(x_0| x_1)\nustar(\rmd x_1) +  \int_{\rset^d} \nabla_{x_0} p^{\rmI}_{1-\delta|0,1}(x_{1-\delta}|x_0, x_1) \pi_{0|1}(x_0| x_1)\nustar(\rmd x_1)}{\int_{\rset^d} p^{\rmI}_{1-\delta|0,1}(x_{1-\delta}|x_0, \tilde{x}_1) \pi_{0|1}(x_0| \tilde{x}_1)\nustar(\rmd \tilde{x}_1)}\\
       &= \int_{\rsetd}\nabla_{x_0}\log\pi_{0|1}(x_0| x_1) p^{\rmI}_{1|0, 1-\delta}(\rmd x_1|x_0, x_{1-\delta}) + \int_{\rset^d}\frac{x_{1-\delta}-\delta x_0 -(1-\delta)x_1}{2(1-\delta)} p^{\rmI}_{1|0, 1-\delta}(\rmd x_1|x_0, x_{1-\delta}) \eqsp,
    \end{split}
\end{equation}
and that
\begin{equation}
    \begin{split}
         \frac{\nabla_{x_{1-\delta}}\pi_{1-\delta}}{\pi_{1-\delta}}(x_0, x_{1-\delta})&= \frac{  \int_{\rset^d} \nabla_{x_{1-\delta}} p^{\rmI}_{1-\delta|0,1}(x_{1-\delta}|x_0, x_1) \pi_{0|1}(x_0| x_1)\nustar(\rmd x_1)}{\int_{\rset^d} p^{\rmI}_{1-\delta|0,1}(x_{1-\delta}|x_0, \tilde{x}_1) \pi_{0|1}(x_0| \tilde{x}_1)\nustar(\rmd \tilde{x}_1)}\\
         &=-\int_{\rset^{d}}\frac{x_{1-\delta}-\delta x_0 -(1-\delta)x_1}{2\delta(1-\delta)}p^{\rmI}_{1|0, 1-\delta}(\rmd x_1|x_0, x_{1-\delta}) \eqsp.
    \end{split}
\end{equation}But then, if we use Jensen inequality and \eqref{bound_moment_hat}, we get
\begin{equation} 
\begin{split}
    \int_{\rset^{2d}} \norm{\nabla_{x_0}\log \frac{\rmd \pi_{1-\delta}}{\rmd \Leb^{2d}}}^8 \rmd \pi_{1-\delta}& \lesssim  \PE\Bigg[\norm{\PE\Bigg[\nabla_{x_0}\log\pi_{0|1}(X^{\rmI}_0| X^{\rmI}_{1})\Bigg|(X^{\rmI}_0, X^{\rmI}_{1-\delta})\Bigg]}^8\Bigg] \\
    &\quad + \PE\Bigg[\norm{\PE\Bigg[\frac{X^{\rmI}_{1-\delta}-\delta X^{\rmI}_0 -(1-\delta)X^{\rmI}_1}{1-\delta}\Bigg|(X^{\rmI}_0, X^{\rmI}_{1-\delta})\Bigg]}^8\Bigg] \\
    &\lesssim \norm{\nabla\log\pi_{0|1}}^8_{\mrl^8(\pi_{0|1})}+ \PE\Bigg[\norm{\frac{X^{\rmI}_{1-\delta}-\delta X^{\rmI}_0 -(1-\delta)X^{\rmI}_1}{1-\delta}}^8\Bigg] \\
    &\lesssim \norm{\nabla\log\pi_{0|1}}^8_{\mrl^8(\pi_{0|1})} +\moment[8]{\nustar_{1-\delta}}\frac{1}{(1-\delta)^8}+ \moment[8]{\mu}\frac{\delta^8}{(1-\delta)^8}+\moment[8]{\nustar}\\
    &\lesssim \norm{\nabla\log\pi_{0|1}}^8_{\mrl^8(\pi_{0|1})}+ \moment[8]{\mu}\frac{\delta^8}{(1-\delta)^8}+ \moment[8]{\nustar} +d^4\frac{\delta^4}{(1-\delta)^4}\\
    &\lesssim \norm{\nabla\log\pi_{0|1}}^8_{\mrl^8(\pi_{0|1})}+\moment[8]{\mu}\frac{1}{(1-\delta)^8}+ \moment[8]{\nustar}\frac{1}{\delta^8} +d^4\frac{1}{\delta^4(1-\delta)^4}\eqsp.
    \end{split}
\end{equation}
and (similarly)
\begin{equation}
    \begin{split}
        \int_{\rset^{2d}} \norm{\nabla_{x_{1-\delta}}\log \frac{\rmd \pi_{1-\delta}}{\rmd \Leb^{2d}}}^8 \rmd \pi_{1-\delta}& \lesssim \PE\Bigg[\norm{\PE\Bigg[\frac{X^{\rmI}_{1-\delta}-\delta X^{\rmI}_0 -(1-\delta)X^{\rmI}_1}{\delta(1-\delta)}\Bigg| (X^{\rmI}_0, X^{\rmI}_{1-\delta})\Bigg]}^8\Bigg] \\
    &\lesssim  \moment[8]{\mu}\frac{1}{(1-\delta)^8}+ \moment[8]{\nustar}\frac{1}{\delta^8} +d^4\frac{1}{\delta^4(1-\delta)^4}\eqsp.
    \end{split}
\end{equation}
Plugging these estimates in the convergence bound provided in \Cref{theo_no_early} and using the fact that $\moment[8]{\mu}, \moment[8]{\nustar}\lesssim d^4$ allows to conclude.
\end{proof}
\begin{proof}[Proof of \Cref{cor:fourier}:] It follows from the very definition of Fortet- Mourier distance, triangle inequality and Pinsker's inequality that  
\begin{align}\label{bound:first_cor}
\wasserstein_{2,\text{FM}}^2(\nustar,\nu^{\thetas}_{1-\delta} )\lesssim \wasserstein_{2}^2(\nustar ,\nustar_{1-\delta} )+ \mathrm{TV}^2(\nustar_{1-\delta}, \nu^{\theta^\star}_{1-\delta})\lesssim\wasserstein_{2}^2(\nustar,\nustar_{1-\delta} )+\KL(\nustar_{1-\delta} | \nu^{\theta^\star}_{1-\delta})\eqsp.
\end{align}with $\mathrm{TV}$ denoting the total variation distance. It follows from \eqref{interpolant_in_linear_form}, that
\begin{align}
    \wasserstein_{2}^2(\nustar,\nustar_{1-\delta} )\lesssim \delta^2\moment[2]{\mu}+\delta^2\moment[2]{\nustar}+\delta(1-\delta)d\eqsp.
\end{align}
Plugging this and the bound provided in \Cref{theo_early} into \eqref{bound:first_cor} leads to 
\begin{align}
    \wasserstein_{2,\text{FM}}^2(\nustar,\nu^{\thetas}_{1-\delta} )\lesssim  \delta^2\moment[2]{\mu}+\delta^2\moment[2]{\nustar}+\delta d+ \varepsilon^2 + h (h^{1/2}+1) \Bigg(\frac{d^2}{\delta^4} +\norm{\nabla\log\pi_{0|1}}^4_{\mrl^8(\pi_{0|1})} \Bigg)d\eqsp.
\end{align}
Therefore, if $\delta = \bigO( \varepsilon^2/d)$, when choosing $h= \bigO(\varepsilon^{10}/d^{7})$, we have that
\begin{align}
    \wasserstein_{2,\text{FM}}^2(\nustar,\nu^{\thetas}_{1-\delta} )\lesssim  \bigO(\varepsilon^2)\eqsp,
\end{align}which proves the desired bound. 

\end{proof}

\subsection{Early-stopping regime with novel step-size schedule}
\begin{proof}[Proof of \Cref{theo_faster}:] Our aim is to apply \Cref{theo_no_early} on the sub-partition $\{t_k\}_{k=0}^{M_h}$ of $[0,1/2]$ with constant step size $h_k=h$ for $k\le M_h$, and Proposition 5.1 in \cite{liu2025finite} on the sub-partition $\{t_k\}_{k=M_h}^{M_h+N}$ of $[1/2,1]$ with exponentially decreasing step sizes $h_k=h\min\{t_k, 1-t_k\}$ for $M_h<k\le M_h+N$. Since the hypothesis of Proposition 5.1 in \cite{liu2025finite} are satisfied in our setting, we immediately get that
\begin{align}\label{eq:first_part_expo}
    \KL(\nustar|\nu^{\theta^\star}_{1})&\lesssim \KL(\nustar_{1/2}|\nu^{\theta^\star}_{1/2}) + \varepsilon^2 +h^2d^3+h^2d^3\log\frac{1}{\delta} +hd^2+hd^2 \log\frac{1}{\delta}\\
    &=\KL(\nustar_{1/2}|\nu^{\theta^\star}_{1/2}) + \varepsilon^2 +h^2d^3\left(1+\log\frac{1}{\delta}\right)+ hd^2\left(1+\log\frac{1}{\delta}\right)\\
    &=\KL(\nustar_{1/2}|\nu^{\theta^\star}_{1/2}) + \varepsilon^2 +\left(1+\log\frac{1}{\delta}\right)\left(h^2d^3+ hd^2\right)
    \\
    &\lesssim \KL(\nustar_{1/2}|\nu^{\theta^\star}_{1/2}) + \varepsilon^2 +hd^3\left(1+\log\frac{1}{\delta}\right)\\
    &\lesssim \KL(\nustar_{1/2}|\nu^{\theta^\star}_{1/2}) + \varepsilon^2 + h d^3\log\frac{1}{\delta}\eqsp,
\end{align}with $\nu^{\theta^\star}_{1/2}$ denoting the law of $\app_{1/2}.$ In order to apply \Cref{theo_no_early}, we first need to prove that the probability distribution $\nu^{\star}_{1/2}=\mathcal{L}\left(X^{\rmI}_{1/2}\right)$ and the coupling 
\begin{align}
    \pi_{0, 1/2}(x_0, x_{1/2})=\int_{\rset^{d}} p^{\rmI}_{1/2|0,1}(x_{1/2}|x_0, x_1) \pi_{0|1}(x_0| x_1)\nustar(\rmd x_1)\in\Pi(\mu, \nu^{\star}_{1/2})\eqsp,
\end{align}with $(x_0,x_1,x_{1/2}) \mapsto p^{\rmI}_{1/2|0,1}(x_{1/2}|x_0, x_1)$ denoting the density of $X^\rmI_{1/2}$ given $(X^\rmI_0, X^\rmI_1)$ with respect to the Lebesgue measure, satisfy $\moment[8]{\nustar}<+\infty,$ and
\begin{align}
    \|\nabla\log \pi_{0, 1/2}\|_{\mrl^8(\pi_{0, 1/2})}<+\infty\eqsp,
\end{align}respectively. 
On the one side, as a very consequence of \eqref{interpolant_in_linear_form}, we have that
\begin{align}\label{bound_moment_hat}
    \moment[8]{\nustar_{1/2}}= \PE\left[\norm{X^{\rmI}_{1/2}}^8\right]\lesssim  \moment[8]{\mu}+\moment[8]{\nustar}+ d^4<+\infty\eqsp.
\end{align}
On the other side, being
\begin{align}
    &\nabla_{x_0} \pi_{0, 1/2}(x_0, x_{1/2})\\
    &=\int_{\rset^{d}} \nabla_{x_0} p^{\rmI}_{1/2|0,1}(x_{1/2}|x_0, x_1) \pi_{0|1}(x_0| x_1)\nustar(\rmd x_1)+ \int_{\rset^{d}} p^{\rmI}_{1/2|0,1}(x_{1/2}|x_0, x_1) \nabla_{x_0}\pi_{0|1}(x_0| x_1)\nustar(\rmd x_1)\eqsp,
\end{align}we have that
\begin{align}
    &\nabla_{x_0} \log \pi_{0, 1/2}(x_0, x_{1/2})\\
    &= \frac{\int_{\rset^{d}} \nabla_{x_0}\log  p^{\rmI}_{1/2|0,1}(x_{1/2}|x_0, x_1) p^{\rmI}_{1/2|0,1}(x_{1/2}|x_0, x_1)\pi_{0|1}(x_0| x_1)\nustar(\rmd x_1)}{\int_{\rset^{d}} p^{\rmI}_{1/2|0,1}(x_{1/2}|x_0, \tilde{x}_1) \pi_{0|1}(x_0| \tilde{x}_1)\nustar(\rmd \tilde{x}_1)}\\
    &\quad +\frac{\int_{\rset^{d}}  \nabla_{x_0}\log \pi_{0|1}(x_0| x_1) p^{\rmI}_{1/2|0,1}(x_{1/2}|x_0, x_1) \pi_{0|1}(x_0| x_1)\nustar(\rmd x_1)}{\int_{\rset^{d}} p^{\rmI}_{1/2|0,1}(x_{1/2}|x_0, \tilde{x}_1) \pi_{0|1}(x_0| \tilde{x}_1)\nustar(\rmd \tilde{x}_1)}\eqsp.
\end{align}
Since
\begin{align}
    \frac{p^{\rmI}_{1/2|0,1}(x_{1/2}|x_0, x_1)\pi_{0|1}(x_0| x_1)\nustar(\rmd x_1)}{\int_{\rset^{d}} p^{\rmI}_{1/2|0,1}(x_{1/2}|x_0, \tilde{x}_1) \pi_{0|1}(x_0| \tilde{x}_1)\nustar(\rmd \tilde{x}_1)}=p^{\rmI}_{1|0, 1/2}(\rmd x_1|x_0, x_{1/2})\eqsp,
\end{align}
we can rewrite $\nabla_{x_0}\log\pi_{0,1/2}$ as 
\begin{align}\label{eq:to_plug_in}
     \nabla_{x_0} \log \pi_{0, 1/2}(x_0, x_{1/2})&= \int_{\rset^{d}} \nabla_{x_0}\log  p^{\rmI}_{1/2|0,1}(x_{1/2}|x_0, x_1) p^{\rmI}_{1|0, 1/2}(\rmd x_1|x_0, x_{1/2})\\
    &\quad +\int_{\rset^{d}}  \nabla_{x_0}\log \pi_{0|1}(x_0| x_1) p^{\rmI}_{1|0, 1/2}(\rmd x_1|x_0, x_{1/2})\eqsp.
\end{align}
It follows from the very definition of the stochastic interpolant that,
\begin{equation}\label{eq:conditional_density_interpolant_mid_interval}
    p^{\rmI}_{1/2|0,1}(x_{1/2}|x_0, x_1)= \frac{1}{\pi ^{d/2}} \exp\Bigg(-\norm{x_{1/2}-\frac{1}{2} x_0 -\frac{1}{2}x_1}^2\Bigg)\eqsp.
\end{equation}
Hence
\begin{align}
    \nabla_{x_0} \log p^{\rmI}_{1/2|0,1}(x_{1/2}|x_0, x_1)=-\frac{1}{2} x_0-\frac{1}{2} x_1+x_{1/2}\eqsp.
\end{align}
Plugging this equality in \eqref{eq:to_plug_in}, we get
\begin{align}
    \nabla_{x_0} \log \pi_{0, 1/2}(x_0, x_{1/2})&= \int_{\rset^{d}} \left(-\frac{1}{2} x_0-\frac{1}{2} x_1+x_{1/2} \right) p^{\rmI}_{1|0, 1/2}(\rmd x_1|x_0, x_{1/2})\\
    &\quad +\int_{\rset^{d}}  \nabla_{x_0}\log \pi_{0|1}(x_0| x_1) p^{\rmI}_{1|0, 1/2}(\rmd x_1|x_0, x_{1/2})\eqsp.
\end{align}
Therefore, using Jensen inequality we get that
\begin{align}
   & \|\nabla_{x_0}\log \pi_{0, 1/2}\|_{\mrl^8(\pi_{0, 1/2})}^8\\
   &\lesssim \PE\left[\left\|\PE\left[-\frac{1}{2} X^\rmI_0-\frac{1}{2} X^\rmI_1+X^\rmI_{1/2}\Bigg| (X^\rmI_0, X^\rmI_{1/2})\right]\right\|^8\right]+\PE\left[\left\|\PE\left[\nabla_{x_0}\log \pi_{0|1}(X^\rmI_0|X^\rmI_1)\Bigg| (X^\rmI_0, X^\rmI_{1/2})\right]\right\|^8\right]\\
   &\lesssim \PE\left[\left\|-\frac{1}{2} X^\rmI_0-\frac{1}{2} X^\rmI_1+X^\rmI_{1/2}\right\|^8\right]+\PE\left[\left\|\nabla_{x_0}\log \pi_{0|1}(X^\rmI_0|X^\rmI_1)\right\|^8\right]\eqsp.
\end{align}
But then, leveraging \eqref{bound_moment_hat}, we obtain
\begin{align}
     \|\nabla_{x_0}\log \pi_{0, 1/2}\|_{\mrl^8(\pi_{0, 1/2})}\lesssim \sqrt[8]{\moment[8]{\mu}}+\sqrt[8]{\moment[8]{\nustar}}+ \|\nabla\log \pi_{0|1}\|_{\mrl^8(\pi_{0|1})}\lesssim \sqrt{d}+\|\nabla\log \pi_{0|1}\|_{\mrl^8(\pi_{0|1})}\eqsp.
\end{align}
Similarly, being
\begin{align}
    \nabla_{x_{1/2}} \pi_{0, 1/2}(x_0, x_{1/2})=\int_{\rset^{d}} \nabla_{x_{1/2}} p^{\rmI}_{1/2|0,1}(x_{1/2}|x_0, x_1) \pi_{0|1}(x_0| x_1)\nustar(\rmd x_1)\eqsp,
\end{align}
we have that
\begin{align}
    &\nabla_{x_{1/2}} \log \pi_{0, 1/2}(x_0, x_{1/2})\\
    &=\frac{\int_{\rset^{d}} \nabla_{x_{1/2}}\log p^{\rmI}_{1/2|0,1}(x_{1/2}|x_0, x_1) p^{\rmI}_{1/2|0,1}(x_{1/2}|x_0, x_1)\pi_{0|1}(x_0| x_1)\nustar(\rmd x_1)}{\int_{\rset^{d}} \nabla_{x_{1/2}} p^{\rmI}_{1/2|0,1}(x_{1/2}|x_0, \tilde{x}_1) \pi_{0|1}(x_0| \tilde{x}_1)\nustar(\rmd \tilde{x}_1)}\\
    &= \int_{\rset^{d}} \nabla_{x_{1/2}}\log p^{\rmI}_{1/2|0,1}(x_{1/2}|x_0, x_1) p^{\rmI}_{1|0,1/2}(\rmd x_1|x_0, x_{1/2})\eqsp.
\end{align}
It follows from \eqref{eq:conditional_density_interpolant_mid_interval} that
\begin{align}
     \nabla_{x_{1/2}}\log p^{\rmI}_{1/2|0,1}(x_{1/2}|x_0, x_1) = -2x_{1/2}+x_0+x_1\eqsp.
\end{align}
At this point, proceeding as before, we get
\begin{align}
     \|\nabla_{x_{1/2}}\log \pi_{0, 1/2}\|_{\mrl^8(\pi_{0, 1/2})}\lesssim \sqrt[8]{\moment[8]{\mu}}+\sqrt[8]{\moment[8]{\nustar}}\lesssim \sqrt{d}\eqsp.
\end{align}
Therefore, we have that
\begin{align}
     \|\nabla\log \pi_{0, 1/2}\|_{\mrl^8(\pi_{0, 1/2})}\lesssim \sqrt{d}+ \|\nabla\log \pi_{0|1}\|_{\mrl^8(\pi_{0|1})}\eqsp.
\end{align}
Recalling \Cref{rem:change_of_time_horizon} and applying \Cref{theo_no_early} on the time interval $[0,1/2]$, with prior $\mu$, target $\nustar_{1/2}$ and coupling $\pi_{0, 1/2}$, we get
\begin{align}
    \KL(\nustar_{1/2}|\nu^{\theta^\star}_{1/2})\lesssim \varepsilon^2 +h(h^{1/2}+1)\Big(d^2+ \|\nabla\log \pi_{0|1}\|_{\mrl^8(\pi_{0|1})}^4 \Big)d\eqsp.
\end{align}
Combining the above bound with \eqref{eq:first_part_expo} yields 
\begin{align}
     \KL(\nustar|\nu^{\theta^\star}_{1})\lesssim \varepsilon^2 +h(h^{1/2}+1)\Big(d^2+ \|\nabla\log \pi_{0|1}\|_{\mrl^8(\pi_{0|1})}^4 \Big)d+h d^3\log\frac{1}{\delta}\eqsp. 
\end{align}
\end{proof}
\begin{proof}[Proof of \Cref{cor:fourier_expo}:] It follows from the very definition of Fortet- Mourier distance, triangle inequality and Pinsker's inequality that  
\begin{align}\label{bound:first_cor_expo}
\wasserstein_{2,\text{FM}}^2(\nustar,\nu^{\thetas}_{1-\delta} )\lesssim \wasserstein_{2}^2(\nustar ,\nustar_{1-\delta} )+ \mathrm{TV}^2(\nustar_{1-\delta}, \nu^{\theta^\star}_{1-\delta})\lesssim\wasserstein_{2}^2(\nustar,\nustar_{1-\delta} )+\KL(\nustar_{1-\delta} | \nu^{\theta^\star}_{1-\delta})\eqsp.
\end{align}with $\mathrm{TV}$ denoting the total variation distance. It follows from \eqref{interpolant_in_linear_form}, that
\begin{align}
    \wasserstein_{2}^2(\nustar,\nustar_{1-\delta} )\lesssim \delta^2\moment[2]{\mu}+\delta^2\moment[2]{\nustar}+\delta(1-\delta)d\eqsp.
\end{align}
Plugging this and the bound provided in \Cref{theo_faster} into \eqref{bound:first_cor_expo} leads to 
\begin{align}
    \wasserstein_{2,\text{FM}}^2(\nustar,\nu^{\thetas}_{1-\delta} )\lesssim  \delta^2\moment[2]{\mu}+\delta^2\moment[2]{\nustar}+\delta d+ \varepsilon^2 +h d^3\log\frac{1}{\delta} +h(h^{1/2}+1)\Big(d^2+ \|\nabla\log \pi_{0|1}\|_{\mrl^8(\pi_{0|1})}^4 \Big)d\eqsp.
\end{align}
Therefore, if $\delta = \bigO( \varepsilon^2/d)$, when choosing $h= \tilde{\bigO}(\varepsilon^{2}/d^{3})$, we have that
\begin{align}
    \wasserstein_{2,\text{FM}}^2(\nustar,\nu^{\thetas}_{1-\delta} )\lesssim  \bigO(\varepsilon^2)\eqsp,
\end{align}which proves the desired bound. 
\end{proof}
\section{Convergence Bounds in Wasserstein-2 Distance}\label{sec:proof_wass}
\subsection{Strong Log-Concave and Full Log-Lipschitz Distributions}
For sake of clarity, we first prove \eqref{wass_weak_convergence_bound} under \Cref{ass_strong_regularity}, \ie, we prove the following result.
\begin{theorem}\label{theo_wasserstein_strong_log_concave}Let $\{t_k\}_{k=0}^{N_h}$ be a uniform partition of $[0,1]$ with step size $h=1/N_h>0$.
Under \Cref{ass_drift_approx_wass,ass_moment,ass_score,ass_strong_regularity,ass_hessian}, denoting by $\nu^{\theta^\star}_1$ the law of $\app_1$,  we have that

\begin{align}
    \label{wass_strong_convergence_bound}
    \wasserstein_2(\nustar,\nu^{\theta^\star}_1)
\lesssim \exp\left( \frac{8\sqrt{2} \|\nabla^2 \log \pi\|_{\mrl^2(\pi)}}{\sqrt{\alpha_{\pi}}}\right)
\left(\varepsilon
+\sqrt{h}(h^{1/4}+1)
\sqrt{\left(d^2+\norm{\nabla\log \pi}_{\mrl^8(\pi)}^4\right)d}\right)\eqsp.
\end{align}
\end{theorem}
\begin{proof}[Proof of \Cref{theo_wasserstein_strong_log_concave}:]
Consider the synchronous coupling between $(X^\rmM_t)_{t\in [0,1]}$ and the continuous time interpolation of $(\app_t)_{t\in [0,1]}$ with the same initialization, \ie, use the same Brownian motion to drive the two processes and set $\app_0=X^\rmM_0$. Then, it holds
\begin{align}\label{eq:wass_bounded_by_l2}
    \wasserstein_2(\nustar, \nu^{\theta^{\star}})\le \left\|X^\rmM_{T}-\app_{T}\right\|_{\mrl^2}\eqsp,
\end{align}where, with abuse of notation, we denoted by $(\app_t)_{t\in [0,1]}$ either the process \eqref{model} and its time continuous interpolation. To upper bound the r.h.s. of the above expression, we estimate $\left\|X^\rmM_{t_{k+1}}-\app_{t_{k+1}}\right\|_{\mrl^2}$ by means of $\left\|X^\rmM_{t_{k}}-\app_{t_{k}}\right\|_{\mrl^2}$ and develop the recursion. Fix $k\in \{0,...,N-1\}$. As we considered the synchronous coupling, we have that
\begin{align}\label{eq:big_one_theo_wass}
    &\left\|X_{t_{k+1}}-\app_{t_{k+1}}\right\|_{\mrl^2}\\
    &= \left\|X_{t_{k}}-\app_{t_{k}}+\int_{t_k}^{t_{k+1}} \left\{\fob_t(X_t)-s_{\theta^\star}(t_k, \app_{t_k})\right\}\rmd  t\right\|_{\mrl^2}\\
    &\le \left\|X_{t_{k}}-\app_{t_{k}}\right\|_{\mrl^2}+\sqrt{h}\left(\int_{t_k}^{t_{k+1}}\PE\left[\left\| \left\{\fob_t(X_t)-\fob_{t_k}(X_{t_k})\right\}\right\|^2\right]\rmd  t\right)^{1/2}+h\left\|\fob_{t_k}(X_{t_k})-\fob_{t_k}(\app_{t_k})\right\|_{\mrl^2}\\
    &\quad + h\left\|\fob_{t_k}(\app_{t_k})-s_{\theta^\star}(t_k, \app_{t_k})\right\|_{\mrl^2}\\
    &\le \left\|X_{t_{k}}-\app_{t_{k}}\right\|_{\mrl^2}+h\left\|\fob_{t_k}(X_{t_k})-\fob_{t_k}(\app_{t_k})\right\|_{\mrl^2}+h\varepsilon+\sqrt{h}\left(\int_{t_k}^{t_{k+1}}\PE\left[\left\| \left\{\fob_t(X_t)-\fob_{t_k}(X_{t_k})\right\}\right\|^2\right]\rmd  t\right)^{1/2}\eqsp,
 \end{align}where, in the last inequality, we have used \Cref{ass_drift_approx}. We now focus on the second term appearing in the r.h.s. of the above expression. First note that if $k=0$ then this term is null. Therefore, we can assume $0<k<N$, so that $0<t_k<1$, use \Cref{ass_strong_regularity} and apply \Cref{theo:strong_drift_regularity}. By doing so, we get
\begin{align}
    h\left\|\tilde{\beta}_{t_k}(X^\rmM_{t_k})-\tilde{\beta}_{t_k}(\app_{t_k})\right\|_{\mrl^2}&\le h \frac{4\sqrt{2}}{\sqrt{\alpha_{\pi}}} \|\nabla^2 \log \pi\|_{\mrl^2(\pi)}\frac{1}{\sqrt{t_k(1-t_k)}}\left\|X^\rmM_{t_{k}}-\app_{t_{k}}\right\|_{\mrl^2}\eqsp.
    \end{align}Plugging this into \eqref{eq:big_one_theo_wass}, we get
 
  \begin{align}\label{eq:big_second_theo_wass}
       \left\|X_{t_{k+1}}-\app_{t_{k+1}}\right\|_{\mrl^2}&\le \left\|X_{t_{k}}-\app_{t_{k}}\right\|_{\mrl^2}+h \frac{4\sqrt{2}}{\sqrt{\alpha_{\pi}}} \|\nabla^2 \log \pi\|_{\mrl^2(\pi)}\frac{1}{\sqrt{t_k(1-t_k)}}\left\|X_{t_{k}}-\app_{t_{k}}\right\|_{\mrl^2}\\
    &\quad +h\varepsilon+\sqrt{h}\left(\int_{t_k}^{t_{k+1}}\PE\left[\left\| \left\{\fob_t(X_t)-\fob_{t_k}(X_{t_k})\right\}\right\|^2\right]\rmd  t\right)^{1/2}\eqsp.
 \end{align}

We are left with bounding the last term appearing in the r.h.s.. Proceeding as in the proof of \Cref{theo_no_early}, we obtain 
\begin{align}
   \sqrt{h}\left(\int_{t_k}^{t_{k+1}}\PE\left[\left\| \left\{\fob_t(X_t)-\fob_{t_k}(X_{t_k})\right\}\right\|^2\right]\rmd  t\right)^{1/2}\le \mathbf{c} h\sqrt{h}(h^{1/4}+1)\sqrt{\left(d^2+\norm{\nabla\log \pi}_{\mrl^8(\pi)}^4\right)d}\eqsp,
 \end{align}
for some universal constant $\mathbf{c}>0$. Plugging this bound in \eqref{eq:big_second_theo_wass}, we get
\begin{align}
    \left\|X_{t_{k+1}}-\app_{t_{k+1}}\right\|_{\mrl^2}\le \gamma_k\left\|X_{t_{k}}-\app_{t_{k}}\right\|_{\mrl^2}+ h\sqrt{h}\mathrm{D}+h\varepsilon\eqsp,\quad k\in\{0,...,N-1\}\eqsp,
\end{align}with
\begin{align}\label{def:l2}
    \gamma_k=1+h \frac{4\sqrt{2}}{\sqrt{\alpha_{\pi}}} \|\nabla^2 \log \pi\|_{\mrl^2(\pi)}\frac{1}{\sqrt{t_k(1-t_k)}}\eqsp,
\end{align}and
\begin{align}\label{def:strong_constant_D}
\mathrm{D}=\mathbf{c}(h^{1/4}+1)\sqrt{\left(d^2+\norm{\nabla\log \pi}_{\mrl^8(\pi)}^4\right)d}\eqsp.
\end{align}
Therefore, if we develop the recursion and use the fact that we set $X^\rmM_{0}=\app_{0}$, we get that
\begin{align}\label{eq:development_of_recursion}
    \left\|X^\rmM_{T}-\app_{T}\right\|_{\mrl^2}&\le \left\|X^\rmM_{0}-\app_{0}\right\|_{\mrl^2}\prod_{l=0}^{N-2}\gamma_l+ (h\sqrt{h}\mathrm{D}+h\varepsilon)\sum_{k=0}^{N-1}\prod_{l=k}^{N-2}\gamma_l\\
    &=( \sqrt{h}\mathrm{D}+\varepsilon)\left(h\;\sum_{k=0}^{N-1}\prod_{l=k}^{N-1}\gamma_l\right)\eqsp.
\end{align}
Note that
\begin{align}
  h\; \sum_{k=0}^{N-1}\prod_{l=k}^{N-1}\gamma_l&= h\;\sum_{k=0}^{\lfloor N/2\rfloor}\prod_{l=k}^{\lfloor N/2\rfloor}\gamma_l+h\;\sum_{k=\lfloor N/2\rfloor}^{N-1}\prod_{l=k}^{N-1}\gamma_l\\
   &\le h\;\sum_{k=0}^{\lfloor N/2\rfloor}\prod_{l=k}^{\lfloor N/2\rfloor}\left(1+\sqrt{h} \frac{4\sqrt{2}\;\|\nabla^2 \log \pi\|_{\mrl^2(\pi)}}{\sqrt{\alpha_{\pi}}}\frac{1}{\sqrt{l}}\right)+h\;\sum_{k=\lfloor N/2\rfloor}^{N-1}\prod_{l=k}^{N-1}\left(1+\sqrt{h} \frac{4\sqrt{2}\;\|\nabla^2 \log \pi\|_{\mrl^2(\pi)}}{\sqrt{\alpha_{\pi}}}\frac{1}{\sqrt{l}}\right)\\
   &= h\;\sum_{k=0}^{N-1}\prod_{l=k}^{N-1}\left(1+\sqrt{h}  \frac{4\sqrt{2}\;\|\nabla^2 \log \pi\|_{\mrl^2(\pi)}}{\sqrt{\alpha_{\pi}}}\frac{1}{\sqrt{l}}\right)\le h\;\sum_{k=0}^{N-1}\prod_{l=k}^{N-1} \exp\left(\sqrt{h}  \frac{4\sqrt{2}\;\|\nabla^2 \log \pi\|_{\mrl^2(\pi)}}{\sqrt{\alpha_{\pi}}}\frac{1}{\sqrt{l}}\right)\\
   &=h\;\sum_{k=0}^{N-1}\exp\left(\sqrt{h}  \frac{4\sqrt{2}\;\|\nabla^2 \log \pi\|_{\mrl^2(\pi)}}{\sqrt{\alpha_{\pi}}}\sum_{l=k}^{N-1}\frac{1}{\sqrt{l}}\right)\le h\;\sum_{k=0}^{N-1}\exp\left(\sqrt{hN}  \frac{8\sqrt{2}\;\|\nabla^2 \log \pi\|_{\mrl^2(\pi)}}{\sqrt{\alpha_{\pi}}}\right)\\
   &\le Nh\exp\left( \frac{8\sqrt{2}\;\|\nabla^2 \log \pi\|_{\mrl^2(\pi)}}{\sqrt{\alpha_{\pi}}}\right)= \exp\left( \frac{8\sqrt{2}\;\|\nabla^2 \log \pi\|_{\mrl^2(\pi)}}{\sqrt{\alpha_{\pi}}}\right)\eqsp.
\end{align}
Therefore, we have that
\begin{align}
     \left\|X^\rmM_{T}-\app_{T}\right\|_{\mrl^2}\le \exp\left( \frac{8\sqrt{2}}{\sqrt{\alpha_{\pi}}}\|\nabla^2 \log \pi\|_{\mrl^2(\pi)}\right)( \sqrt{h}\mathrm{D}+\varepsilon)\eqsp.
\end{align}
Recalling the definition \eqref{def:strong_constant_D} of $\mathrm{D}$, and plugging this bound in \eqref{eq:wass_bounded_by_l2}, we obtain \eqref{wass_strong_convergence_bound}.
\end{proof}
\subsection{Weakly Log-Concave and One-Sided Log-Lipschitz Distributions}
\begin{proof}[Proof of \Cref{theo_wasserstein_weak_log_concave}:]
We proceed as in the proof of \Cref{theo_wasserstein_strong_log_concave}, to obtain 
\begin{align}
    \left\|X_{t_{k+1}}-\app_{t_{k+1}}\right\|_{\mrl^2}&\le \left\|X_{t_{k}}-\app_{t_{k}}\right\|_{\mrl^2}+h\left\|\fob_{t_k}(X_{t_k})-\fob_{t_k}(\app_{t_k})\right\|_{\mrl^2}+h\varepsilon+ h\sqrt{h}\mathrm{D}\eqsp,
 \end{align}with $\mathrm{D}$ as in \eqref{def:strong_constant_D}. To bound the second term of the r.h.s. of the above expression, we use \Cref{theo:drift_regularity} instead of \Cref{theo:strong_drift_regularity}. By doing so, we get
 \begin{align}
     h\left\|\fob_{t_k}(X_{t_k})-\fob_{t_k}(\app_{t_k})\right\|_{\mrl^2}\le (\gamma_k-1)  \left\|X_{t_{k}}-\app_{t_{k}}\right\|_{\mrl^2}\eqsp,
 \end{align}with
 \begin{align}\label{eq:gamma_weak}
    \gamma_k=1+
        h\frac{4\sqrt{2}}{\sqrt{\alpha_{\pi}}}\exp\left(\frac{M_\pi}{\alpha_\pi}\right) \|\nabla^2 \log \pi\|_{\mrl^2(\pi)}\frac{1}{\sqrt{t_k(1-t_k)}}\eqsp.
\end{align}
We therefore get a recursive formula of the type
\begin{align}
    \left\|X_{t_{k+1}}-\app_{t_{k+1}}\right\|_{\mrl^2}\le \gamma_k\left\|X_{t_{k}}-\app_{t_{k}}\right\|_{\mrl^2}+ h\sqrt{h}\mathrm{D}+h\varepsilon\eqsp,\quad k\in\{0,...,N-1\}\eqsp,
\end{align}with $\gamma_k=(\gamma_k-1)+1$ and $(\gamma_k-1)$ as in \eqref{eq:gamma_weak} and $\mathrm{D}$ as in \eqref{def:strong_constant_D}. Developing the recursion as before and using \eqref{eq:wass_bounded_by_l2}, we obtain that
\begin{align}\label{eq:weak_development_of_recursion}
     \wasserstein_2(\nu, \applaw)&\le \left\|X_{T}-\app_{T}\right\|_{\mrl^2}\le \left\|X_{0}-\app_{0}\right\|_{\mrl^2}\prod_{l=0}^{N-2}\gamma_l+ (h\sqrt{h}\mathrm{D}+h\varepsilon)\sum_{k=0}^{N-1}\prod_{l=k}^{N-2}\gamma_l\\
    &=( \sqrt{h}\mathrm{D}+\varepsilon)\left(h\;\sum_{k=0}^{N-1}\prod_{l=k}^{N-1}\gamma_l\right)\eqsp.
\end{align}

Proceeding as before, we get
\begin{align}
h\;\sum_{k=0}^{N-1}\prod_{l=k}^{N-1}\gamma_l\le \exp\left( \frac{8\sqrt{2}}{\sqrt{\alpha_{\pi}}}\exp\left(\frac{M_\pi}{\alpha_\pi}\right)\|\nabla^2 \log \pi\|_{\mrl^2(\pi)}\right) 
\end{align}
Plugging this bound in \eqref{eq:weak_development_of_recursion}, and recalling the definition \eqref{def:strong_constant_D} of $\mathrm{D}$ respectively, we obtain \eqref{wass_weak_convergence_bound}.
\end{proof}
\end{document}